\documentclass[accepted]{uai2026} 

\usepackage[american]{babel}

\usepackage{xcolor}
\newif\ifcredits\creditstrue 
\providecommand{\cred}[1]{\ifcredits\textcolor{blue}{#1}\else#1\fi}
\newenvironment{credblock}{\par\ifcredits\color{blue}\fi}{\par} 
\providecommand{\dslp}[2][blue]{{\color{#1}#2}} 

\usepackage{natbib}
\usepackage{mathtools} 
\usepackage{amssymb,amsthm}
\usepackage{bbm}

\usepackage{booktabs}
\usepackage{graphicx}
\usepackage{caption}
\usepackage{etoolbox}
\AtBeginEnvironment{table}{\footnotesize}

\usepackage{algorithm}
\usepackage{algpseudocode}

\newcommand{\E}{\mathbb{E}}
\newcommand{\R}{\mathbb{R}}
\newcommand{\N}{\mathcal{N}}
\newcommand{\eps}{\varepsilon}
\newcommand{\del}{\delta}
\newcommand{\norm}[1]{\left\lVert #1 \right\rVert}
\newcommand{\KL}{\mathrm{KL}}
\newcommand{\ind}{\mathbbm{1}}
\providecommand{\codexedit}[1]{\ifcredits\textcolor{magenta}{#1}\else#1\fi}
\newenvironment{codexblock}{\par\ifcredits\color{magenta}\fi}{\par}
\providecommand{\greenedit}[1]{\ifcredits\textcolor{green!40!black}{#1}\else#1\fi}
\definecolor{newcitehl}{rgb}{1.0,0.86,0.45}
\newcommand{\newcite}[1]{\ifcredits\par\noindent\colorbox{newcitehl}{\parbox{\dimexpr\linewidth-2\fboxsep\relax}{#1}}\par\else#1\fi}

\theoremstyle{definition}
\newtheorem{definition}{Definition}
\newtheorem{assumption}{Assumption}

\theoremstyle{plain}
\newtheorem{theorem}{Theorem}
\newtheorem{lemma}{Lemma}
\newtheorem{proposition}{Proposition}
\newtheorem{corollary}{Corollary}

\theoremstyle{remark}
\newtheorem{remark}{Remark}

\title{Workload-Preserving Differentially Private Synthetic Data\\for Causal Inference via Maximum-Entropy Calibration}

\author[1]{\href{mailto:amir.asiaeetaheri@vumc.org?Subject=Your UAI 2026 paper}{Amir~Asiaee}{}}
\author[2]{Kaveh~Aryan}
\affil[1]{%
    Department of Biostatistics\\
    Vanderbilt University Medical Center\\
    Nashville, Tennessee, USA%
}
\affil[2]{%
    Department of Informatics\\
    King's College London\\
    London, UK%
}

\creditsfalse                       
\renewcommand{\dslp}[2][blue]{#2}   
\colorlet{green}{black}             

\begin{document}
\maketitle

\begin{abstract}
\greenedit{Workload-based differentially private (DP) synthetic data methods privately measure aggregate queries and post-process the noisy answers into synthetic records.}
\greenedit{Generic workloads can achieve strong distributional fidelity, but causal estimands such as the average treatment effect (ATE) depend on treatment-arm balance and outcome moments that generic marginals need not preserve.}
\greenedit{We propose \emph{causal workloads}: DP query sets designed around the orthogonal moments used by doubly robust causal estimators.}
\greenedit{The released workload can be used directly by stable moment-map estimators or reconstructed by maximum-entropy calibration into reusable synthetic data; our theory decomposes ATE error into sampling, privacy, workload-approximation, Monte Carlo, and calibration terms.}
\greenedit{We also introduce \textsc{Causal-AIM}, an adaptive workload selector, and a noise-aware multiple-imputation (NA+MI) procedure for confidence intervals from DP synthetic data.}
\greenedit{Because the workload is released once, the same DP synthetic table can support ATE, ATT, and subgroup analyses without additional privacy spending.}
\codexedit{Empirically, causal workloads are most useful at strict privacy budgets and for calibrated uncertainty, while generic workloads often retain an advantage for point RMSE as privacy relaxes.}
\codexedit{The broader lesson is a tradeoff: distributional fidelity can help point accuracy, but valid causal inference requires preserving causal moments and propagating DP noise rather than treating synthetic rows as real.}
\end{abstract}

\section{Introduction}\label{sec:intro}

Differential privacy is increasingly \greenedit{used for public release} of sensitive tabular datasets in domains such as health, education, and social science \citep{dwork2014foundations}.
A common deployment pattern is to release DP synthetic microdata so that downstream analysts can reuse existing statistical and machine learning workflows \citep{liang_generating_2026}.
State-of-the-art DP synthesis methods often follow a select--measure--reconstruct pipeline: select a set of low-dimensional queries, measure them privately, and reconstruct a distribution (often maximum entropy / graphical model) from the noisy answers \citep{mckenna2021nist,mckenna2022aim}.
However, \emph{causal inference} introduces additional structure and fragility.
Even if synthetic data match many marginals, causal estimands can be biased by subtle distortions in overlap, confounding structure, or conditional outcome models.
Recent work emphasizes that \dslp[red]{DP noise inflates the variance of causal estimators and invalidates naive confidence intervals unless it is explicitly modeled} \citep{farzam2024causal,schroder2025private,dp-cate2025,ohnishi2024covbal}.

This paper asks a concrete question:
\begin{quote}
\emph{Which DP queries must a synthetic data mechanism preserve to enable valid causal inference, and what are the fundamental privacy--causal tradeoffs?}
\end{quote}
\greenedit{We answer by designing \emph{causal workloads} and giving explicit bounds for stable workload-based estimators and for the calibrated synthetic-data route.}

\paragraph{Key idea: causal workloads as orthogonal moments.}
Modern causal estimators, including \codexedit{doubly robust (DR) estimators \citep{robins1994aipw} and double/debiased machine learning \citep{chernozhukov2018dml}}, rely on score functions whose expectation identifies the target estimand.
\greenedit{We propose to choose a workload that directly measures these orthogonal moments, or a basis approximation thereof, under DP.}
\greenedit{The released workload can then be used in two ways: the direct $q$-route evaluates a stable estimator as a function of the released moments, while the synthetic-data route reconstructs a maximum-entropy distribution matched to those moments and samples reusable synthetic records.}
\greenedit{Our theory controls the moment-level error in both routes, with additional calibration and Monte Carlo terms for synthetic-data analysis.}

\paragraph{Why synthetic data at all?}
If the goal is ATE only, one could release a DP ATE estimate directly.
Synthetic data is valuable when many analysts ask many questions, or when we want to preserve a whole \emph{class} of causal estimands (ATE, average treatment effect on the treated (ATT) $\E[Y(1)-Y(0)\mid T{=}1]$, subgroup effects, balancing diagnostics, etc.) without repeated access to the private data.
Causal workloads formalize this ``class of analyses'' viewpoint\greenedit{; Section~\ref{sec:experiments} provides direct evidence that a single release supports ATE, ATT, and subgroup analyses with no additional privacy spending, a reuse property structurally unavailable to direct DP ATE estimators}.

\paragraph{Our contributions.}
{\ifcredits\color{green!40!black}\fi
\begin{enumerate}
\setlength{\itemsep}{1pt}
\setlength{\topsep}{1pt}
\setlength{\parsep}{0pt}
\item We define \emph{causal workloads}: DP query sets built from orthogonal-score moments, with two downstream routes: direct moment plug-ins and maximum-entropy synthetic-data reconstruction.
\item We prove finite-sample error bounds that connect ATE error to released-workload error, and decompose synthetic-data error into sampling, privacy, workload-approximation, Monte Carlo, and calibration terms. All proofs are collected in Appendix~\ref{app:proofs}.
\item We introduce \textsc{Causal-AIM}, an adaptive causal workload selector, and NA+MI, a noise-aware multiple-imputation procedure for confidence intervals from DP synthetic data.
\item Empirically, Causal + NA+MI is the only private method with near-nominal coverage on all four benchmarks ($99.8{-}100\%$ at $\eps\le1$), while naive synthetic-data analyses remain at or below $35.2\%$; the same release also supports ATE, ATT, and subgroup analyses without extra privacy spending.
\end{enumerate}
}

\section{Related Work}\label{sec:related}

\begin{credblock}
\paragraph{A brief DP-synthesis primer.}
A mechanism is $(\eps,\del)$-differentially private if replacing one record changes its output distribution by at most a factor $e^{\eps}$, up to slack $\del$ \citep{dwork2014foundations}.
Workload-based \codexedit{DP synthetic-data methods first choose a set of aggregate queries (the \emph{workload}), release noisy answers to those queries, and then reconstruct a distribution whose query answers match the noisy measurements.}
\codexedit{The final synthetic records are sampled from this reconstructed distribution; reconstruction and sampling are post-processing, so they spend no additional privacy budget.}
\codexedit{This select--measure--reconstruct template underlies MWEM (Multiplicative Weights Exponential Mechanism; \citealp{hardt2012mwem}), which iteratively fixes poorly matched queries; Private-PGM, which fits a graphical model to noisy low-order marginals \citep{mckenna2019privatepgm}; MST (Maximum-Spanning-Tree marginal selection with Private-PGM reconstruction; \citealp{mckenna2021nist}), which won the NIST synthetic-data challenge; and AIM (Adaptive and Iterative Mechanism; \citealp{mckenna2022aim}), which adaptively selects marginals against a target workload.}
\end{credblock}

\paragraph{Workload-based DP synthetic data.}
A large class of methods measures a set of marginals / low-dimensional queries and reconstructs a distribution \codexedit{by imposing structure such as a graphical-model factorization or an optimization objective over candidate distributions}, including Private-PGM \citep{mckenna2019privatepgm} and its NIST-MST/MST instantiations \citep{mckenna2021nist}, the adaptive AIM mechanism \citep{mckenna2022aim}, and recent optimal-transport based approaches \citep{donhauser2024privpgd}.
Benchmarks suggest these methods can outperform GAN-style synthesis on tabular data for many metrics \citep{tao2022benchmark,bowen2020comparative}.
These methods are designed for general-purpose fidelity; none explicitly targets causal estimands.

\paragraph{Inference from DP synthetic data.}
Naively treating DP synthetic data as real can produce invalid inference, e.g.\ inflated type-I error \citep{perez2024mwutest}.
\codexedit{Noise-aware procedures} combine Bayesian modeling of the DP noise with multiple imputation to recover calibrated uncertainty \citep{raisa2023noiseaware}.
\codexedit{For record-level synthetic microdata---tables whose rows represent individual units rather than only aggregate summaries---classical multiple-imputation inference dates back to} \citet{reiter2003inference} and \citet{raghunathan2003multiple}.
Our noise-aware multiple imputation (NA+MI) procedure builds on this line but specializes it to causal estimands by choosing workloads informed by the orthogonal score structure.

\paragraph{Causal inference under DP.}
Recent papers propose DP causal estimators with point and interval estimation guarantees.
PrivATE \citep{schroder2025private} develops doubly robust DP confidence intervals for ATE using output perturbation.
\citet{ohnishi2024covbal} propose DP covariate balancing with finite-sample guarantees.
DP-CATE \citep{dp-cate2025} extends orthogonal learning to heterogeneous treatment effects under DP.
\citet{farzam2024causal} study causal inference under DP broadly \codexedit{and analyze how DP mechanisms can inflate ATE variance and distort individual-level effects}.
Earlier, \citet{niu2022dp_cate} proposed DP estimation of heterogeneous causal effects, and concurrent work by \citet{lebeda2025model_agnostic} develops a model-agnostic DP causal framework.
Our work \codexedit{focuses instead} on \emph{synthetic data} as the release object: rather than releasing a single DP estimate, we release a DP synthetic dataset that supports a whole class of causal analyses.

\paragraph{Double/debiased machine learning.}
Our causal workload design draws on the orthogonal moment framework of \citet{chernozhukov2018dml}, which shows that Neyman-orthogonal scores enable root-$n$ inference for treatment effects when nuisance functions are estimated at slower rates.
We \codexedit{use this structure to choose} workloads that preserve the moments needed for orthogonal scores.

\newcite{
  \paragraph{Causal prior for synthetic data.}
  A related line encodes causal or analysis-specific prior knowledge \emph{into the synthetic generation process itself}: reward-guided generation steers a generator toward a target downstream utility such as regression-coefficient recovery \citep{jackson2026reward}, and post-hoc constraint wrappers inject partial causal structure, such as trusted or forbidden edges and monotonicity, into arbitrary tabular generators \citep{asiaee2026causalwrap}. Our objective is different: rather than encoding a prior during generation, we design the DP measurements so that the released dataset preserves a target causal estimand, and we account for the DP noise in downstream inference.}

\section{Problem Setup}\label{sec:setup}

We observe i.i.d.\ data $D=\{(X_i,T_i,Y_i)\}_{i=1}^n$, where $T\in\{0,1\}$ is treatment.
The target is the average treatment effect (ATE)
\(
\tau := \E[ Y(1)-Y(0)].
\)

\begin{assumption}[Unconfoundedness and positivity]
\label{assump:causal}
$(Y(1),Y(0))\perp T \mid X$ and the propensity score \citep{rosenbaum1983propensity} $e(x):=\Pr(T=1\mid X=x)$ satisfies $e(x)\in[\eta,1-\eta]$ almost surely for some $\eta>0$.
\end{assumption}

Under Assumption~\ref{assump:causal},
\(\tau = \E[m_1(X)-m_0(X)]\), where \(m_t(x):=\E[Y\mid T=t,X=x]\).

\section{Causal Workloads and Maximum-Entropy Calibration}\label{sec:method}

\cred{Figure~\ref{fig:pipeline} gives an overview of the full release-and-inference procedure.}

\subsection{A Workload That Targets Orthogonal Moments}

Let $\phi:\mathcal{X}\to\R^p$ be a feature map (e.g.\ spline basis, random Fourier features, tree-based leaves, or indicators for a discretization of $X$).
Define the \emph{moment vectors}
\begin{align} \label{eq:1}
q_{t}^{(0)}(D) &:= \frac{1}{n}\sum_{i=1}^n \ind\{T_i=t\}\,\phi(X_i) \in \R^p, \\ \label{eq:2}
q_{t}^{(1)}(D) &:= \frac{1}{n}\sum_{i=1}^n \ind\{T_i=t\}\,Y_i\,\phi(X_i) \in \R^p.
\end{align}
\codexedit{Here, $q_t^{(0)}$ records treatment-arm feature masses, and $q_t^{(1)}$ records outcome-weighted feature moments.}
\codexedit{As a simple bin-indicator example, each coordinate of $q_t^{(1)}$ is the joint arm-bin mass times the mean outcome in that bin, with the mass scaled by $1/n$.}
\codexedit{These four blocks are the private measurements used in the experimental pipeline.}
\codexedit{Second-order Gram moments are useful for some direct regression plug-ins, but they are not part of the default release; Appendix~\ref{app:feature-encodings} gives the feature-encoding example and the extension.}

\begin{definition}[Causal workload]
Fix a feature map $\phi$.
\codexedit{The \emph{causal workload} $\mathcal{Q}_\phi$ used in our main procedure is the ordered collection of query functions whose answers are $\{q_{0}^{(0)},q_{1}^{(0)},q_{0}^{(1)},q_{1}^{(1)}\}$.}
\codexedit{Its stacked answer is}
\[
\codexedit{q(D):=\bigl[q_0^{(0)}(D),q_1^{(0)}(D),q_0^{(1)}(D),q_1^{(1)}(D)\bigr]\in\R^{4p}.}
\]
\codexedit{Here ``workload'' means the collection of DP queries, while $q(\cdot)$ denotes the realized answer vector.}
\end{definition}

\codexedit{This base $4p$ workload is what we release in the experiments. It is \emph{one principled choice}, not a claimed optimum: it targets the first-order treatment and outcome moments used by orthogonal-score inference on nuisances projected onto the $\phi$ basis (Proposition~\ref{prop:sufficiency}, Appendix~\ref{app:sufficiency}), without asserting minimax optimality, uniqueness, or a lower bound.}

\begin{credblock}
\begin{figure}[t]
\centering
\includegraphics[width=\columnwidth]{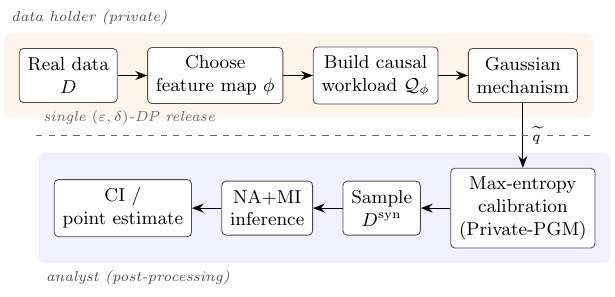}
\caption{End-to-end pipeline. The data holder chooses the feature map $\phi$, builds the causal workload, and releases Gaussian-mechanism-noised moments once. \codexedit{Downstream analysis is post-processing: one may use the noisy moments directly in a stable moment-map estimator, or reconstruct a maximum-entropy distribution (via Private-PGM), sample synthetic data, and run NA+MI inference.}}
\label{fig:pipeline}
\end{figure}
\end{credblock}

\subsection{Private Measurement of the Workload}

\codexedit{Let $W=(X,T,Y)$ and let $m=4p$ be the dimension of the causal workload $q(D)$ defined above.}
\codexedit{For $a \in [m]$, let $h_a(W_i)$ denote record $i$'s contribution to coordinate $a$, so $q_a(D)=n^{-1}\sum_i h_a(W_i)$; for any distribution $P$ over $W$, $q_a(P)=\E_P h_a(W)$.}
\codexedit{We use $q(D)$ for the exact, non-released empirical answer on the confidential data; $q^\star=q(P^\star)$ for the population answer under the true data-generating law $P^\star$; and $\widetilde q$ for the DP noisy release.}
\codexedit{The fixed causal mechanism measures all $m$ coordinates once with the Gaussian mechanism:}
\[
\codexedit{\widetilde q=q(D)+Z,\qquad Z\sim\N(0,\sigma^2 I_m).}
\]
\codexedit{Assume $Y$ is clipped to $[-B,B]$ and $\norm{\phi(X)}_2\le \phi_{\max}$.}
\codexedit{Let $\Delta_2:=\sup_{D\sim D'}\norm{q(D)-q(D')}_2$ be the replacement-adjacency $\ell_2$ sensitivity.}
\(
\codexedit{\Delta_2 \le \frac{2\phi_{\max}\sqrt{1+B^2}}{n}.}
\)
\codexedit{The standard Gaussian mechanism \citep{dwork2014foundations} uses}
\(
\codexedit{
\sigma=\frac{\Delta_2\sqrt{2\log(1.25/\del)}}{\eps}.}
\)
\codexedit{For an adaptive mechanism, the same notation is applied to the measured coordinate set $S\subset[m]$: $\widetilde q_S$ denotes the stored noisy answers on $S$.}

\subsection{Two Routes from the Released Workload for Causal Estimation}

\codexedit{The DP object produced by measurement is a noisy aggregate vector, not yet a synthetic dataset.}
\codexedit{It can be used in two ways.}
\codexedit{In the \emph{direct moment route} (or $q$-route), an estimator that can be written as a stable function $\widehat\tau_{\mathrm{est}}(q)$ is evaluated directly at $\widetilde q$ (or at the measured coordinates $\widetilde q_S$); Section~\ref{sec:theory} analyzes this route because it exposes how moment error becomes causal-estimation error.}
\codexedit{In the \emph{synthetic-data route}, we reconstruct a distribution $P^{\mathrm{syn}}$ whose workload answers match the noisy measurements, sample synthetic records from that distribution, and then run standard causal estimators such as \cred{DR / augmented inverse-propensity weighting (AIPW)} and NA+MI.}
\codexedit{The experiments use this synthetic-data route; the theory controls the part of this route that is expressible through the reconstructed workload answer $q(P^{\mathrm{syn}})$, plus finite synthetic-sampling error.}
\codexedit{The exponential-family form below is precisely the reconstruction step: it turns noisy moments into a distribution, and sampling from that distribution produces the synthetic data.}
\codexedit{We present each route in two versions: a fixed mechanism that measures all queries at once ($S=[m]$), and an adaptive mechanism, \textsc{Causal-AIM}, that builds $S$ over rounds.}

\subsection{Maximum-Entropy Reconstruction for Synthetic Data}

Let $\mathcal{P}$ be a class of distributions over $W=(X,T,Y)$ (discrete support or parametric).
\codexedit{Given a measured coordinate set $S\subset[m]$ (with $S=[m]$ for the fixed full workload), we define the calibrated synthetic distribution as an information projection:}
{\small
\begin{equation}
\label{eq:iprojection}
P^{\mathrm{syn}}
=
\arg\min_{P\in\mathcal{P}}
\KL(P\ \|\ P_0)
\ \text{s.t.} \
\codexedit{\forall a \in S,\ q_a(P)=\widetilde{q}_a,}
\end{equation}} 
where \codexedit{$P_0$ is a reference distribution, such as a uniform distribution on the discretized support, an independent product of noisy one-way marginals, or a public population prior.}
In the discrete case, the solution is an exponential family:
\[
\codexedit{p^{\mathrm{syn}}(w)\propto p_0(w)\exp\Big(\sum_{a\in S} \xi_a h_a(w)\Big)}
\]
for Lagrange multipliers \codexedit{$\xi$}.
\codexedit{Appendix~\ref{app:maxent-details} gives the derivation and the relaxed form used when noisy constraints are infeasible.}
This is the same mathematical object used in Private-PGM \citep{mckenna2019privatepgm,mckenna2021nist}.
\codexedit{Thus reconstruction means fitting the maximum-entropy distribution $p_\xi$ whose measured workload answers match $\widetilde q_S$ up to the DP noise scale; synthetic-data generation is the subsequent sampling of records from $p_\xi$.}
\codexedit{The fixed synthesis procedure is given in Appendix~\ref{app:maxent-details}.}

\subsection{Adaptive Query Selection: \textsc{Causal-AIM}}

AIM selects queries adaptively to reduce worst-case error in a generic workload \citep{mckenna2022aim}.
\codexedit{The fixed mechanism measures the whole causal workload once, so $S=[m]$.}
\codexedit{\textsc{Causal-AIM} is the adaptive variant: it builds $S$ over rounds by selecting the next feature group to measure using a private estimate of causal utility.}
\codexedit{Here a candidate $\phi_r$ can be a single feature or a group such as all bins of one covariate; let $I(r)\subset[m]$ denote the corresponding coordinates of the released workload.}
At iteration $k$, \textsc{Causal-AIM} \codexedit{(i) starts from the current measured set $S_{k-1}$ and its reconstructed model $P_{k-1}$, (ii) evaluates a private upper bound on ATE error using an orthogonal score residual $\psi(W;\theta,\zeta)$, where $\psi$ is the ATE influence function and $\zeta$ collects nuisance functions, and (iii) selects the next feature group $\phi_{r_k}$, or its corresponding moment coordinates $I(r_k)$, before refitting to obtain $P_k$; Appendix~\ref{app:causal-aim} gives details.}
\codexedit{Selected groups are removed from the candidate set: once $I(r)$ has been measured, its noisy answer is reused in later refits rather than measured again.}
\codexedit{After any round, the stored noisy vector $\widetilde q_{S_k}$ can be used directly in a moment-map estimator or passed through maximum-entropy reconstruction to generate synthetic data; Algorithm~\ref{alg:causalaim} displays the synthetic-data route used in our experiments.}
This yields a causal-utility-driven workload that is usually much smaller than ``all marginals up to order 2''.

\begin{algorithm}[t]
\caption{\textsc{Causal-AIM} (high-level)}
\label{alg:causalaim}
\begin{algorithmic}[1]
\Require \codexedit{Data $D$, privacy budget $(\eps,\del)$, feature groups $\mathcal R$, iterations $K$, reference distribution $P_0$}
\State \codexedit{Initialize selected coordinates $S_0\leftarrow\emptyset$, available set $A_0\leftarrow\mathcal R$, and current reconstruction $P^{(0)}\leftarrow P_0$}
\For{\codexedit{$k=1,\dots,\min(K,|\mathcal R|)$}}
    \State \codexedit{Compute an internal utility $u_k(r)$ for each unmeasured group $r\in A_{k-1}$, estimating the ATE-error reduction from adding coordinates $I(r)$}
    \State \codexedit{Select $r_k$ using the exponential mechanism applied to $\{u_k(r):r\in A_{k-1}\}$, then update $S_k\leftarrow S_{k-1}\cup I(r_k)$ and $A_k\leftarrow A_{k-1}\setminus\{r_k\}$}
    \State \codexedit{Measure only the new coordinates $I(r_k)$ with the Gaussian mechanism and store their noisy answers}
    \State \codexedit{Refit $P_k$ by solving \eqref{eq:iprojection} on all stored noisy coordinates $S_k$}
\EndFor
\State \codexedit{Output stored noisy moments $\widetilde q_{S_K}$; for synthetic release, output $P_K$ and sample $D^{\mathrm{syn}}\sim P_K$ (post-processing)}
\end{algorithmic}
\end{algorithm}

\section{Theory: Causal Error Bounds and Tradeoffs}\label{sec:theory}

\codexedit{This section tracks how error in the released workload affects ATE estimation along the two routes.}
\codexedit{We first analyze direct moment plug-ins, where the estimator is a function of the released workload, and then add the reconstruction-calibration and finite-synthetic-sample terms needed for the synthetic-data route.}
\codexedit{Thus the theoretical object is an estimator with an explicit workload argument, $\widehat\tau_{\mathrm{est}}(q)$; this is narrower than an arbitrary downstream learner run on synthetic rows.}
\codexedit{We state the main decomposition for the ATE; the same weighted-estimand argument gives ATT and subgroup effects when the corresponding weights are represented by the workload (Appendix~\ref{app:weighted-ate}).}

\subsection{Bounding ATE Error by Moment Error}

\codexedit{The two routes in Section~\ref{sec:method} share the same first question: how accurately does DP release the workload moments?}
\codexedit{For the direct moment route, the second question is how sensitive an ATE estimator is to perturbing those moments.}
\codexedit{For the second question, call a direct moment-map estimator $\widehat\tau_{\mathrm{est}}(q)$ \emph{$L_{\mathrm{est}}$-stable} on a workload domain if}
\[
\codexedit{
|\widehat\tau_{\mathrm{est}}(q)-\widehat\tau_{\mathrm{est}}(q')|
\le
L_{\mathrm{est}}\norm{q-q'}_\infty
\quad\forall \text{admissible }q,q'.}
\]
\codexedit{Here $q$ may be the population answer $q^\star=q(P^\star)$, the empirical answer $q(D)$, the DP release $\widetilde q$, or the reconstructed answer $q(P^{\mathrm{syn}})$.}
\codexedit{The formal stability examples below analyze the direct route; the later decomposition composes the same stability logic with the calibration and Monte Carlo terms that appear when the released workload is routed through maximum-entropy reconstruction and synthetic sampling.}
\codexedit{This separates the generic moment-level guarantee from software-specific choices about the reconstruction solver and downstream nuisance fitting.}

\begin{credblock}
\paragraph{Example 1: projected ridge plug-in.}
\codexedit{Some direct regression plug-ins require the arm-specific Gram moment $G_t(D):=n^{-1}\sum_i\ind\{T_i=t\}\phi(X_i)\phi(X_i)^\top$; write $q_t^{(2)}(D):=\operatorname{vec}\{G_t(D)\}$.}
\codexedit{For a workload-answer vector containing the needed blocks, write $G_t(q)$ for the Gram block, $r_t(q)=q_t^{(1)}$, and $\bar\phi(q)=q_0^{(0)}+q_1^{(0)}$.}
\codexedit{Here ``projection'' means the ridge least-squares approximation of the arm-specific outcome regression $m_t(x)$ in the span of $\phi$.}
\codexedit{The arm-$t$ coefficient is $\widehat\beta_{t,\lambda}(q)=\{G_t(q)+\lambda I\}^{-1}r_t(q)$, with corresponding projected \cred{conditional average treatment effect (CATE)} and ATE:}
\begin{align}
\label{eq:projected-cate}
\codexedit{\widehat\tau_\lambda(x;q)}
&\codexedit{:=\phi(x)^\top\{\widehat\beta_{1,\lambda}(q)-\widehat\beta_{0,\lambda}(q)\},}\\
\label{eq:projected-ate}
\codexedit{\widehat\tau_\lambda(q)}
&\codexedit{:=\bar\phi(q)^\top\{\widehat\beta_{1,\lambda}(q)-\widehat\beta_{0,\lambda}(q)\}.}
\end{align}

\begin{theorem}[Projected ridge moment stability]
\label{thm:lipschitz}
\codexedit{Assume $Y\in[-B,B]$ and $\norm{\phi(X)}_2\le \phi_{\max}$.}
\codexedit{Fix $\lambda>0$ and use $\widehat\tau_\lambda(q)$ as defined in \eqref{eq:projected-ate}.}
\codexedit{Suppose that for each $q_\circ\in\{q,q'\}$ and $t\in\{0,1\}$, $G_t(q_\circ)\succeq\kappa I$ for some $\kappa\ge0$ and $\norm{\widehat\beta_{t,\lambda}(q_\circ)}_2\le R_\beta$.}
\codexedit{Then for two admissible workload-answer vectors $q,q'$ of the same type (same coordinate ordering and same rule for $G_t$),}
{\small 
\[
\codexedit{
|\widehat{\tau}_\lambda(q)-\widehat{\tau}_\lambda(q')|
\;\le\;
L_\phi \norm{q-q'}_\infty,\
L_\phi \le C_\phi\left(1+\frac{1}{\kappa+\lambda}\right).}
\]}
\codexedit{The constant $C_\phi$ depends on $(\phi_{\max},R_\beta,p)$ but not on $n$, $\eps$, or the particular DP noise draw; Appendix~\ref{app:proof-lipschitz} gives one explicit choice.}
\end{theorem}

\begin{remark}[Ridge and rank]
\codexedit{The useful quantity is $\kappa+\lambda$: ridge compensates for small eigenvalues of the arm-specific Gram matrices.}
\codexedit{If $\lambda=0$, the same statement requires $\kappa>0$; if $\kappa=0$, one must take $\lambda>0$ to make the inverse stable.}
\end{remark}

\paragraph{Example 2: clipped IPW/AIPW moment maps.}
\codexedit{Inverse-propensity weighting (IPW) estimates $\E[Y(1)]$ and $\E[Y(0)]$ by reweighting observed treated and control outcomes by inverse treatment probabilities.}
\codexedit{With partition-cell features, those reweighted sums can be written directly as functions of the released moments.}
\codexedit{For partition-cell features, let $p_{tj}(q)=q_{tj}^{(0)}$, $r_{tj}(q)=q_{tj}^{(1)}$, $s_j(q)=p_{0j}(q)+p_{1j}(q)$, and define the clipped cell propensity $\widehat e_j(q)=\operatorname{clip}\{p_{1j}(q)/s_j(q),\eta,1-\eta\}$.}
\codexedit{The direct clipped IPW moment map is}
\[
\codexedit{
\widehat\tau_{\mathrm{IPW}}(q)
=
\sum_{j}
\left\{
\frac{r_{1j}(q)}{\widehat e_j(q)}
-
\frac{r_{0j}(q)}{1-\widehat e_j(q)}
\right\}.}
\]
\codexedit{AIPW adds the usual outcome-regression augmentation to the same cell-level IPW score; Appendix~\ref{app:ipw-aipw-stability} gives the full expression.}
\begin{proposition}[Clipped IPW/AIPW moment stability]
\label{prop:aipw-stability}
\codexedit{For partition-cell features, bounded outcomes, lower-bounded cell masses, and propensities clipped to $[\eta,1-\eta]$, the direct moment-map IPW and AIPW estimators are $L_{\mathrm{IPW}}$- and $L_{\mathrm{AIPW}}$-stable.}
\codexedit{The constants scale polynomially in $1/\eta$ and the inverse minimum cell mass; Appendix~\ref{app:ipw-aipw-stability} gives explicit bounds.}
\codexedit{This is standard clipped-score stability \citep{rosenbaum1983propensity,robins1994aipw,chernozhukov2018dml}; the paper-specific step is plugging the resulting $L_{\mathrm{est}}$ into the DP moment-release decomposition.}
\end{proposition}

\codexedit{For either example, setting $q=q(D)$ and $q'=\widetilde q$ converts the deterministic stability bound into a privacy-noise bound.}
\codexedit{For ridge, the following corollary also records the conditioning requirement for the noisy Gram blocks.}

\begin{corollary}[Noise-calibrated ridge]
\label{cor:noise_ridge}
\codexedit{Suppose the workload coordinates entering $\widehat\tau_\lambda(q)$ have independent Gaussian noise with coordinate scale $\sigma$, and set $\delta_m=c\sigma\sqrt{\log(m/\gamma)}$.}
\codexedit{On the event $\norm{\widetilde q-q(D)}_\infty\le\delta_m$ and $\max_t\norm{\widetilde G_t-G_t}_{\mathrm{op}}\le\lambda/2$, the noisy ridge matrices remain well conditioned and}
\[
\codexedit{
|\widehat\tau_\lambda(\widetilde q)-\widehat\tau_\lambda(q(D))|
\lesssim
C_\phi\left(1+\frac{1}{\kappa+\lambda}\right)\delta_m.}
\]
\codexedit{For full joint-cell bases with diagonal Gram matrices, the operator-norm condition follows from $\lambda\gtrsim\delta_m$; for explicitly released dense Gram blocks, standard Gaussian-matrix bounds give the analogous sufficient choice $\lambda\gtrsim\sigma\sqrt{p+\log(1/\gamma)}$.}
\end{corollary}

\codexedit{Here and below, SNR means the coordinatewise signal-to-noise ratio $|\widetilde q_a|/\sigma_a$.}
\codexedit{Appendix~\ref{app:conditioning-ridge} gives the conditioning intuition and explains how ridge complements the SNR calibration filter.}
\end{credblock}

\codexedit{Theorem~\ref{thm:lipschitz} and Proposition~\ref{prop:aipw-stability} are two examples of the same moment-stability template: ridge covers direct regression plug-ins that may require Gram moments, while IPW/AIPW covers cell-based causal scores computable from the base treatment and outcome moments.}

\subsection{Accuracy of the DP Moment Release}

\codexedit{The previous subsection is deterministic once a moment-error bound is available.}
\codexedit{The next result supplies that bound for the Gaussian mechanism: it converts the sensitivity of the released workload into the coordinate error used by the stability examples and the decomposition below.}

\begin{theorem}[DP moment accuracy]
\label{thm:dp_moment}
\codexedit{Let $q(D)=n^{-1}\sum_i h(W_i)\in\R^m$ be the released workload answer, and suppose $\norm{h(W)}_2\le H_q$ for every record.}
\codexedit{Measuring $q(D)$ with the Gaussian mechanism under $(\eps,\del)$-DP yields, with probability at least $1-\gamma$,}
{\small
\[
\codexedit{\norm{\widetilde{q}-q(D)}_\infty
=
O\!\left(\frac{H_q\sqrt{\log(m/\gamma)\log(1/\del)}}{n\eps}\right).}
\]}
\codexedit{For the base causal workload, $H_q=\phi_{\max}\sqrt{1+B^2}$.}
\codexedit{If optional Gram blocks are appended for a direct regression plug-in, one may instead take $H_q=\{\phi_{\max}^2(1+B^2)+\phi_{\max}^4\}^{1/2}$.}
\end{theorem}

\emph{Remark on norms.}
\cred{Theorems~\ref{thm:lipschitz} and~\ref{thm:dp_moment} both use $\norm{\phi(X)}_2\le \phi_{\max}$; indicator features additionally satisfy $\norm{\phi(X)}_\infty\le 1$, and for $p$-dimensional indicator features both hold simultaneously with $\phi_{\max}\le\sqrt{p}$.}

\subsection{Overall ATE Error Decomposition}

\codexedit{This subsection does not introduce a third route; it composes the direct-route moment-stability bound with the two extra errors introduced when the released workload is converted into synthetic microdata.}
\codexedit{At the population-reconstruction level, the controlled estimand is $\widehat\tau_{\mathrm{est}}\{q(P^{\mathrm{syn}})\}$.}
\codexedit{When this quantity is estimated from $n_{\mathrm{syn}}$ sampled synthetic records, write the result as $\widehat{\tau}^{\mathrm{syn}}$ and treat finite synthetic sample size as an additional sampling term.}
\codexedit{The experiments instantiate this release--reconstruct--analyze path with DR/AIPW on synthetic microdata; the theorem formalizes the workload-matching error budget that this pipeline is designed to control.}

\begin{credblock}
Because the iterative max-entropy solver satisfies the noisy constraints only approximately, the bound carries an explicit calibration term.
\codexedit{For a measured or retained coordinate set $S$ (with $S=[m]$ for the fixed full workload when no thresholding is used), let $\Pi_S$ denote coordinate projection: $\Pi_S v=(v_a:a\in S)$.}
\codexedit{In implementation, $\Pi_S$ is the Boolean mask over retained moment coordinates.}
\codexedit{Define}
\[
\codexedit{\mathrm{CalGap}(S,\sigma) := L_{\mathrm{est}}\,\bigl\lVert \Pi_S\{\widetilde q - q(P^{\mathrm{syn}})\}\bigr\rVert_2,}
\]
where $\codexedit{L_{\mathrm{est}}}$ is the appropriate estimator Lipschitz constant ($\codexedit{L_\phi}$ for the ridge plug-in of Theorem~\ref{thm:lipschitz}, or the IPW/AIPW constants in Proposition~\ref{prop:aipw-stability} and Appendix~\ref{app:ipw-aipw-stability}) and $q(P^{\mathrm{syn}})$ the workload moments of the solver output.

\begin{theorem}[ATE error decomposition]
\label{thm:ate_bound}
\codexedit{Under Assumption~\ref{assump:causal}, the boundedness conditions above, and an estimator Lipschitz condition $|\widehat\tau_{\mathrm{est}}(q)-\widehat\tau_{\mathrm{est}}(q')|\le L_{\mathrm{est}}\norm{q-q'}_\infty$ for the retained workload coordinates, let $\widehat\tau^{\mathrm{syn}}$ be the empirical estimate of $\widehat\tau_{\mathrm{est}}\{q(P^{\mathrm{syn}})\}$ formed from $n_{\mathrm{syn}}$ i.i.d.\ draws from the solver output $P^{\mathrm{syn}}$.}
\codexedit{This condition holds for the ridge plug-in by Theorem~\ref{thm:lipschitz} and for partition-feature clipped IPW/AIPW by Proposition~\ref{prop:aipw-stability}. Then}
{\small
\begin{align} \nonumber
|\widehat{\tau}^{\mathrm{syn}}-\tau|
&\le
\underbrace{\codexedit{O_p\!\bigl(L_{\mathrm{est}}n^{-1/2}\bigr)}}_{\text{sampling}} 
+
\underbrace{\codexedit{O_p\!\!\left(\frac{L_{\mathrm{est}}H_q\sqrt{\log m\,\log(\frac{1}{\del})}}{n\eps}\right)}}_{\text{privacy}}\\ \nonumber
&+
\underbrace{\mathrm{Approx}(\phi;S)}_{\text{workload}}
\;+\;
\underbrace{O_p\!\bigl(n_{\mathrm{syn}}^{-1/2}\bigr)}_{\text{Monte Carlo}}
+
\underbrace{\mathrm{CalGap}(S,\sigma)}_{\text{calibration}},
\end{align}}
where $\mathrm{Approx}(\phi;S)$ is the approximation error of representing $m_t(x)$ and $e(x)$ using the retained workload coordinates.
When $P^{\mathrm{syn}}$ satisfies the constraints exactly, the fifth term vanishes and the bound reduces to the four-term decomposition.
\end{theorem}

\begin{remark}[SNR thresholding controls the calibration gap]\label{rem:calibration_gap}
\codexedit{Retaining only high-signal coordinates $S_\tau=\{a: |\widetilde q_a|/\sigma_a\ge\tau_{\mathrm{SNR}}\}$ and running the solver to noise-level tolerance $|\widetilde q_a-q_a(P^{\mathrm{syn}})|\le c_{\mathrm{cal}}\sigma_a$ on $S_\tau$ yields}
$\codexedit{\mathrm{CalGap}(S_\tau,\sigma)\le L_{\mathrm{est}} c_{\mathrm{cal}}\bar\sigma\sqrt{m_{\mathrm{kept}}}}$,
where $\codexedit{m_{\mathrm{kept}}}=|S_\tau|$ and $\bar\sigma$ bounds the per-coordinate noise scale (Corollary~\ref{cor:snr_calgap}, Appendix~\ref{app:proof-ate}).
Discarded moments are not free: their effect moves into $\mathrm{Approx}(\phi;S_\tau)$, so thresholding trades a smaller, controllable calibration gap against possible approximation bias.
Empirically, which side of the tradeoff binds is data-dependent: in our workload-dimension ablation on IHDP (\cred{Figure~\ref{fig:ablation_dim}}), the variance side dominates---\cred{without thresholding, richer workloads reduce RMSE throughout the tested range, and thresholding lowers bias but costs more variance than it saves}---while wider workloads or stricter budgets push toward the bias-dominated regime the threshold guards against.
Because $\mathrm{CalGap}$ is computable from the release (the solver's residual is recorded per run), the tradeoff can be monitored rather than assumed.
\end{remark}
\end{credblock}

\begin{corollary}[Design rule for causal utility]
\label{cor:design}
To keep DP distortion below sampling noise, it suffices (up to logs) that
{\small
\begin{align*}
\codexedit{\frac{H_q\sqrt{\log m}}{n\eps} \ll \frac{1}{\sqrt{n}} 
\Longleftrightarrow \eps \gg H_q\sqrt{\frac{\log m}{n}}.}
\end{align*}}
This recovers a natural ``$\eps$ vs.\ $n$'' threshold analogous to parametric DP inference \citep{smith2011privacy}, with the outcome bound $B$ controlling the difficulty.
\end{corollary}

\section{Uncertainty Quantification from DP Synthetic Data}\label{sec:uq}

{\color{green}Recent work shows} that synthetic-data inference must account for both sampling variability and DP noise \citep{raisa2023noiseaware,perez2024mwutest}.
We propose a causal version of NA+MI:
\begin{enumerate}
  \item Treat the measured DP moments $\codexedit{\widetilde{q}_S}$ as noisy observations of latent true moments $\codexedit{q_S}$, with $\codexedit{S=[m]}$ for the fixed full workload and adaptive $S$ for \textsc{Causal-AIM}.
  \item \cred{Sample $M$ posterior draws $\codexedit{q_S^{(\ell)}\sim\N(\widetilde q_S,\Sigma_S)}$, $\codexedit{\ell}=1,\dots,M$: the asymptotic Gaussian posterior for $q_S$ given $\widetilde q_S$ under a flat prior, with covariance given by the known Gaussian-mechanism noise on the measured coordinates.}
  \item For each draw, construct $\codexedit{P^{\mathrm{syn},(\ell)}}$ via \eqref{eq:iprojection} on the measured coordinates and sample a synthetic dataset $\codexedit{D^{\mathrm{syn},(\ell)}}$.
  \item Compute $\codexedit{\widehat{\tau}^{(\ell)}}$ on each synthetic dataset and combine using Rubin's MI rules \citep{rubin1987mi}.
\end{enumerate}
\cred{Appendix~\ref{app:bias-aware} gives a bias-aware variant that {\color{green}inflates} the MI variance by an estimated workload-approximation term, with a coverage guarantee; Appendix~\ref{app:na-mi} states the full procedure as pseudocode.}
The result is a confidence interval that widens as $\eps$ decreases (stronger privacy), similar in spirit to \citet{raisa2023noiseaware}.

\section{Experiments}\label{sec:experiments}

\codexedit{We evaluate three questions: whether causal workloads improve ATE accuracy over generic workloads, whether NA+MI yields calibrated confidence intervals, and when adaptive \textsc{Causal-AIM} improves over a fixed workload.}

\subsection{Setup}

{\ifcredits\color{green!40!black}\fi
\noindent\textbf{Benchmarks.}
We use five benchmark settings where the target ATE is known, so both point error and coverage can be evaluated directly.
\begin{itemize}
\setlength{\itemsep}{1pt}
\setlength{\topsep}{1pt}
\setlength{\parsep}{0pt}
  \item \codexedit{\textbf{IHDP} \citep{hill2011bayesian}: $n=747$ children from the Infant Health and Development Program benchmark, with 25 child and family covariates. Treatment indicates the program intervention, the observed outcome is a continuous cognitive score, and the standard semi-simulated release supplies potential-outcome means for the ground-truth ATE.}
  \item \codexedit{\textbf{Twins} \citep{louizos2017causal}: $n=11{,}400$ same-sex twin-pair records with 30 birth and maternal covariates. The benchmark constructs a confounded treatment indicator by sampling heavier-versus-lighter twin status from a covariate-dependent propensity; the outcome is one-year mortality, with both potential mortality outcomes observed within each pair.}
  \item \codexedit{\textbf{ACIC 2016} \citep{dorie2019acic}: $n=4{,}802$ units from the competition covariate matrix, with anonymized real covariates and three categorical fields one-hot encoded into 82 numeric columns. We use DGP~7, which simulates confounded treatment assignment and continuous potential outcomes with heterogeneous treatment effects; the ATE is computed from the exported potential outcomes.}
  \item \codexedit{\textbf{LaLonde/NSW} \citep{lalonde1986evaluating}: the National Supported Work experimental sample with $n=445$ subjects (185 treated and 260 controls). Covariates include age, education, race, marital status, and pre-program earnings; the outcome is 1978 earnings, and the experimental difference-in-means is the reference effect.}
  \item \codexedit{\textbf{ACS}: California 2018 American Community Survey person microdata from \texttt{folktables} \citep{ding2021retiring}. We use 20 demographic covariates and simulate confounded treatment and continuous potential outcomes at $n\in\{1000,5000,20000\}$, so the ATE is known by construction.}
\end{itemize}

\noindent\textbf{Baselines.}
All private methods release synthetic data and then estimate the ATE on synthetic records using doubly robust analysis.
\begin{itemize}
\setlength{\itemsep}{1pt}
\setlength{\topsep}{1pt}
\setlength{\parsep}{0pt}
  \item \textbf{Non-private DR}: an oracle doubly robust estimator with cross-fitting on the confidential data.
  \item \textbf{MST + naive DR}: generic MST-style synthesis \citep{mckenna2021nist}, followed by standard DR inference that ignores DP noise.
  \item \textbf{AIM + naive DR}: generic AIM synthesis \citep{mckenna2022aim}, followed by the same naive DR analysis.
  \item \textbf{Causal workload + naive DR}: our fixed causal workload mechanism, analyzed without the NA+MI uncertainty correction.
  \item \textbf{Causal workload + NA+MI}: the fixed causal workload mechanism with noise-aware multiple imputation.
  \item \textbf{\textsc{Causal-AIM} + NA+MI}: adaptive causal workload selection with the same NA+MI analysis.
\end{itemize}

\noindent\textbf{Defaults and metrics.}
Default runs use $L=5$ quantile bins, $M=20$ MI draws, $n_{\mathrm{syn}}=10n$, clipped standardized outcomes, DR analysis on synthetic data, and no SNR thresholding or calibration ridge in the main pipeline\cred{; we keep $n_{\mathrm{syn}}=10n$ for comparability across methods, although the ablation shows $n_{\mathrm{syn}}=n$ already suffices}.
We report ATE RMSE and empirical 95\% coverage over $\eps\in\{0.5,1,2,5\}$ with $\delta=1/n^2$, and run ablations over workload dimension, MI draws, synthetic sample size, overlap strength, and adaptive rounds $K$.
Appendix~\ref{app:exp-details} gives feature-construction details, metric definitions, and ablation grids; additional figures appear in Appendix~\ref{app:figures}.
}

\subsection{Results}

\codexedit{Table~\ref{tab:main-results} summarizes all methods at $\eps{=}1$, Figures~\ref{fig:exp1_rmse} and~\ref{fig:exp2_coverage} show RMSE and coverage across privacy budgets, and Figure~\ref{fig:exp3_adaptive} shows the adaptive fixed-workload comparison on ACIC.}

\begin{table*}[t]
\caption{ATE RMSE and 95\% CI Coverage at $\eps=1.0$, $\del=1/n^2$. Each cell is RMSE / Coverage; among private methods, bold marks the best value and underlining marks the second-best value in each dataset column.}
\centering
\begin{tabular}{lcccc}
\toprule
Method & IHDP & Twins & ACIC & LaLonde \\
\midrule
Non-private DR & 0.259/0.940 & 0.007/1.000 & 0.236/1.000 & 47.593/1.000 \\
MST + naive & \underline{20.100}/0.014 & \textbf{0.023}/0.352 & \textbf{0.583}/0.236 & \textbf{3290.027}/0.118 \\
AIM + naive & 46.863/0.002 & 0.168/0.064 & 6.206/0.060 & 21546.320/0.008 \\
Causal + naive & 43.942/0.024 & 0.291/0.078 & 16.338/0.016 & 14707.613/0.048 \\
Causal + NA+MI & \textbf{15.074}/\textbf{0.998} & 0.219/\textbf{1.000} & 8.034/\textbf{1.000} & \underline{6163.055}/\textbf{1.000} \\
Causal-AIM + NA+MI & 47.633/\underline{0.346} & \underline{0.150}/\underline{0.956} & \underline{2.981}/\underline{0.718} & 18147.736/\underline{0.678} \\
\bottomrule
\end{tabular}
\label{tab:main-results}
\end{table*}

\codexedit{\textbf{Generic workloads win on RMSE at loose budgets, but lose on coverage.}
MST + naive DR achieves the lowest private RMSE on all four benchmarks at $\eps\ge2$ and on three of four at $\eps=1$.
At strict budgets, causal workload + NA+MI becomes more competitive, matching or beating MST on two of four benchmarks at $\eps=0.5$ and on IHDP at $\eps=1$.
However, MST's point accuracy comes with invalid uncertainty quantification: at $\eps\le1$, all naive methods have coverage at most $35.2\%$ (Figure~\ref{fig:exp2_coverage}).}

\codexedit{\textbf{NA+MI restores calibrated inference.}
Causal workload + NA+MI is the only private method with near-nominal coverage at $\eps\le1$ on all four benchmarks ($99.8{-}100\%$; Table~\ref{tab:main-results}).
The price is wider intervals: NA+MI intervals are \cred{$34$--$300\times$ wider than naive intervals at strict budgets (Figure~\ref{fig:exp2_length})} because they account for DP measurement noise.
This is the intended behavior when privacy noise dominates the sampling error; the narrow naive intervals are overconfident rather than genuinely informative.}

\codexedit{\textbf{Adaptive workloads are operating-point choices.}
On ACIC at $\eps\ge1$, \textsc{Causal-AIM} substantially reduces RMSE relative to the fixed causal workload, especially at small $K$ (Figure~\ref{fig:exp3_adaptive}).
At $\eps=0.5$, additional adaptive rounds can hurt because each round consumes privacy budget that may exceed the gain from better feature targeting.
On IHDP, no tested $K$ recovers the fixed workload's calibration (Appendix~\ref{app:promoted}), so fixed causal workload + NA+MI remains the conservative default when coverage is the priority.}

\codexedit{\textbf{Robustness and reuse.}
On the ACS semi-synthetic study, NA+MI again provides the coverage advantage: at $(n=1000,\eps=0.5)$, Causal + NA+MI has coverage $1.00$ while MST + naive DR has coverage $0.07$, and NA+MI also wins RMSE in that cell ($0.88$ vs.\ $1.17$).
Ablations show that $M=20$ MI draws is conservative, synthetic sample sizes above $n_{\mathrm{syn}}=n$ add little, and SNR thresholding is mainly useful as a safeguard when moments approach the noise floor \cred{(Figures~\ref{fig:ablation_dim}, \ref{fig:ablation_mi}, \ref{fig:ablation_nsyn}, and~\ref{fig:ablation_overlap})}.
Finally, one DP synthetic release supports multiple estimands: ATE, ATT, and a subgroup effect all achieve nominal coverage in the reuse experiment (Appendix~\ref{app:promoted}).}

\codexedit{\textbf{Fidelity is not causal utility.}
Generic workloads dominate marginal-fidelity metrics such as \cred{average marginal total-variation distance (TVD)} in nearly every configuration \cred{(Figure~\ref{fig:fidelity_tradeoff})}, yet those metrics do not rank methods by ATE error or coverage.
This is the empirical counterpart of the workload argument: preserving generic low-dimensional marginals is not enough when the downstream target depends on treatment-arm outcome and balance moments.}

\codexedit{\textbf{Practical takeaway.}
Use Causal + NA+MI when valid uncertainty quantification is the primary goal, especially at strict privacy budgets.
Use \textsc{Causal-AIM} when point RMSE is prioritized and its operating point is favorable for the dataset.
Use generic workloads when the release is intended for broad marginal fidelity or exploratory analysis and calibrated causal intervals are not required.}


\begin{figure}[t]
\centering
\includegraphics[width=\columnwidth]{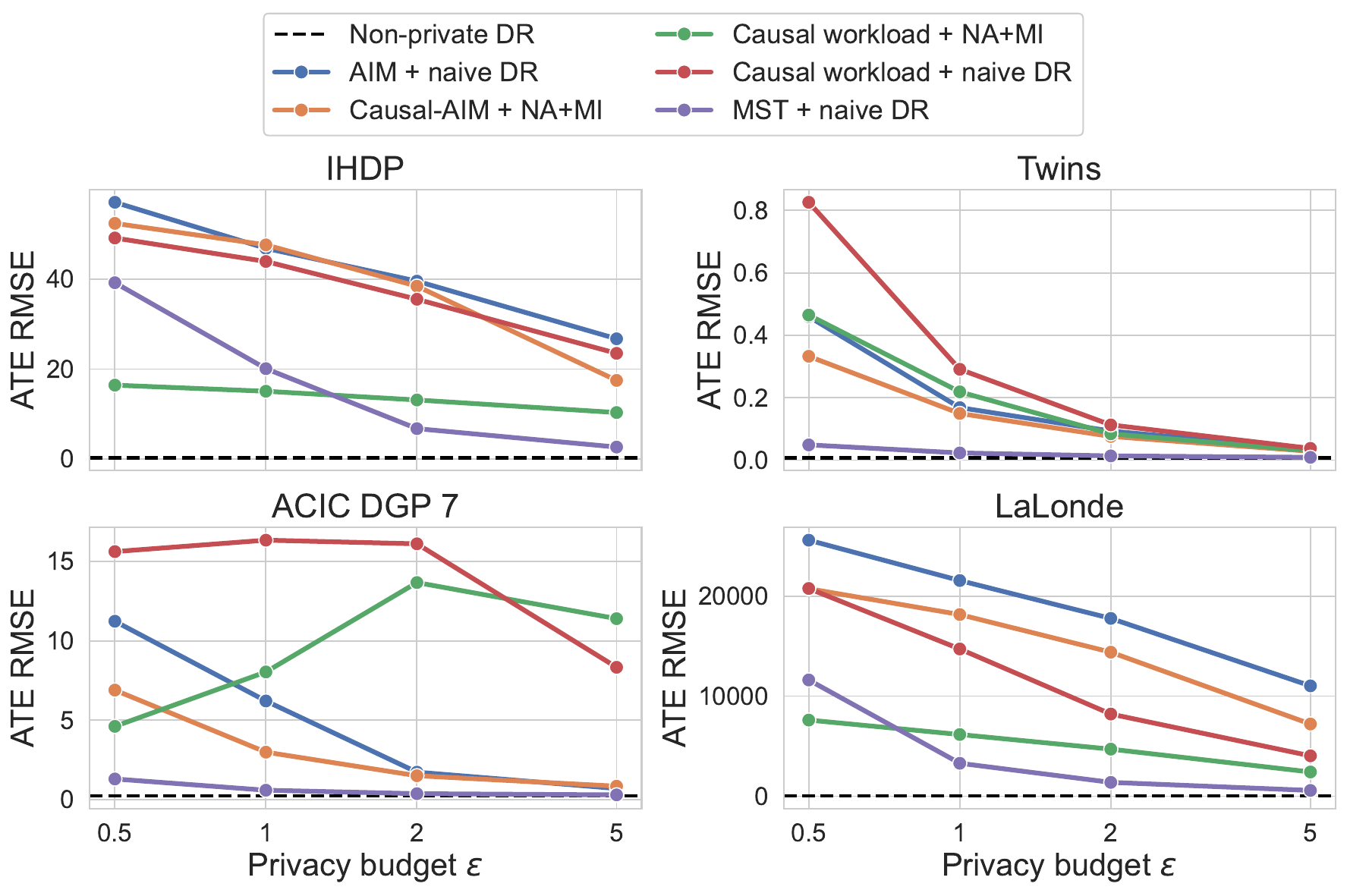}
\caption{\cred{ATE RMSE versus privacy budget ($\eps$) on four benchmarks ($n_{\mathrm{rep}}=500$, $\delta=1/n^2$). MST + naive DR leads on RMSE at $\eps\ge 2$; at $\eps=0.5$, Causal + NA+MI matches or beats it on half the benchmarks, and on IHDP at $\eps=1$. The value of causal workload methods lies chiefly in enabling calibrated coverage (Figure~\ref{fig:exp2_coverage}).}}
\label{fig:exp1_rmse}
\end{figure}

\begin{figure}[t]
\centering
\includegraphics[width=\columnwidth]{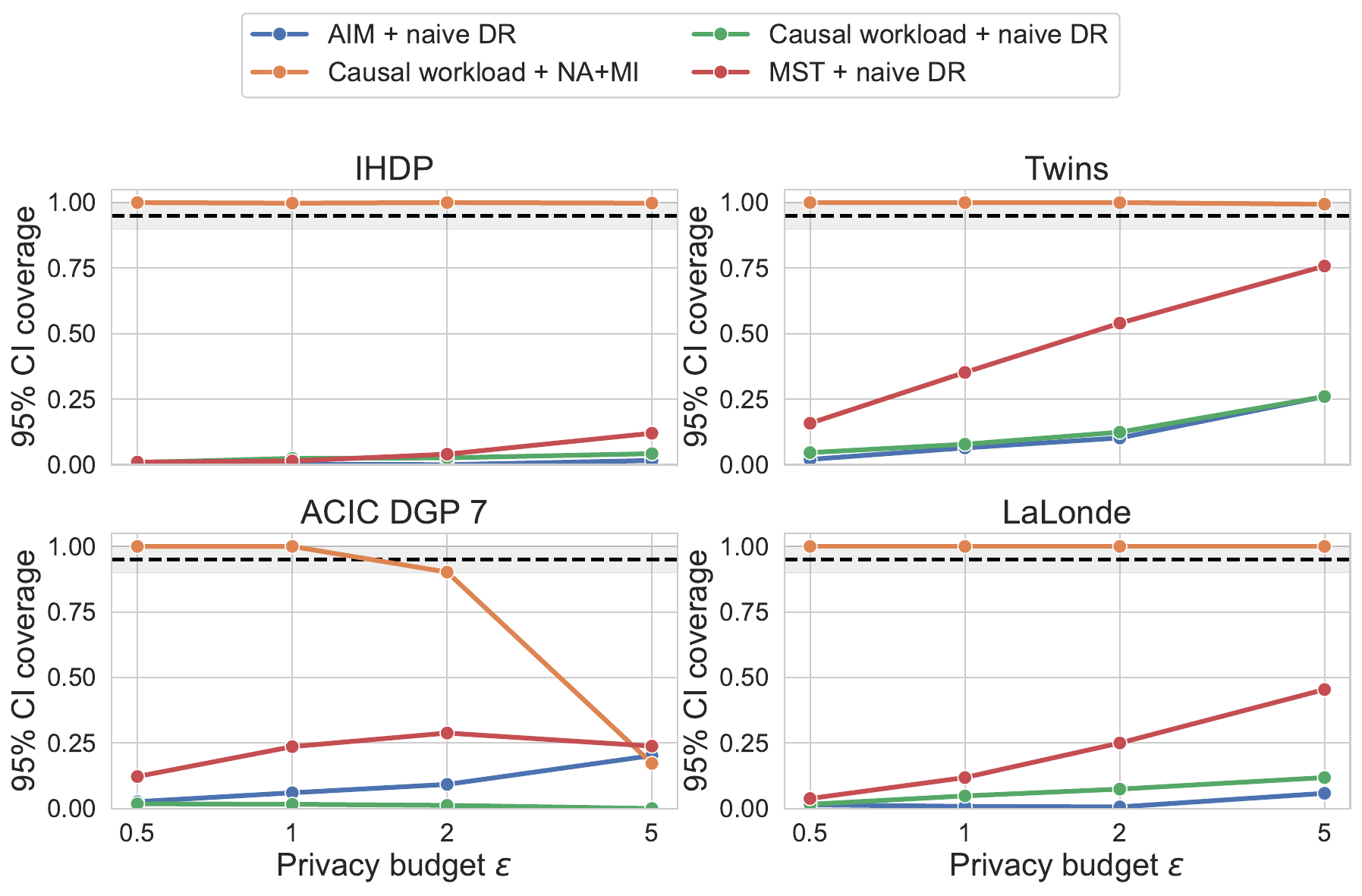}
\caption{\cred{Empirical 95\% CI coverage across privacy budgets and four datasets ($n_{\mathrm{rep}}=500$). The dashed line marks nominal 0.95. All naive methods are at or below $35.2\%$ at $\eps\le 1$. Causal workload + NA+MI achieves $99.8{-}100\%$ coverage at $\eps \le 1$ on all four benchmarks; Coverage can decline at high $\eps$ due to approximation bias (Appendix~\ref{app:discussion}).}}
\label{fig:exp2_coverage}
\end{figure}

\begin{figure}[t]
\centering
\includegraphics[width=\columnwidth]{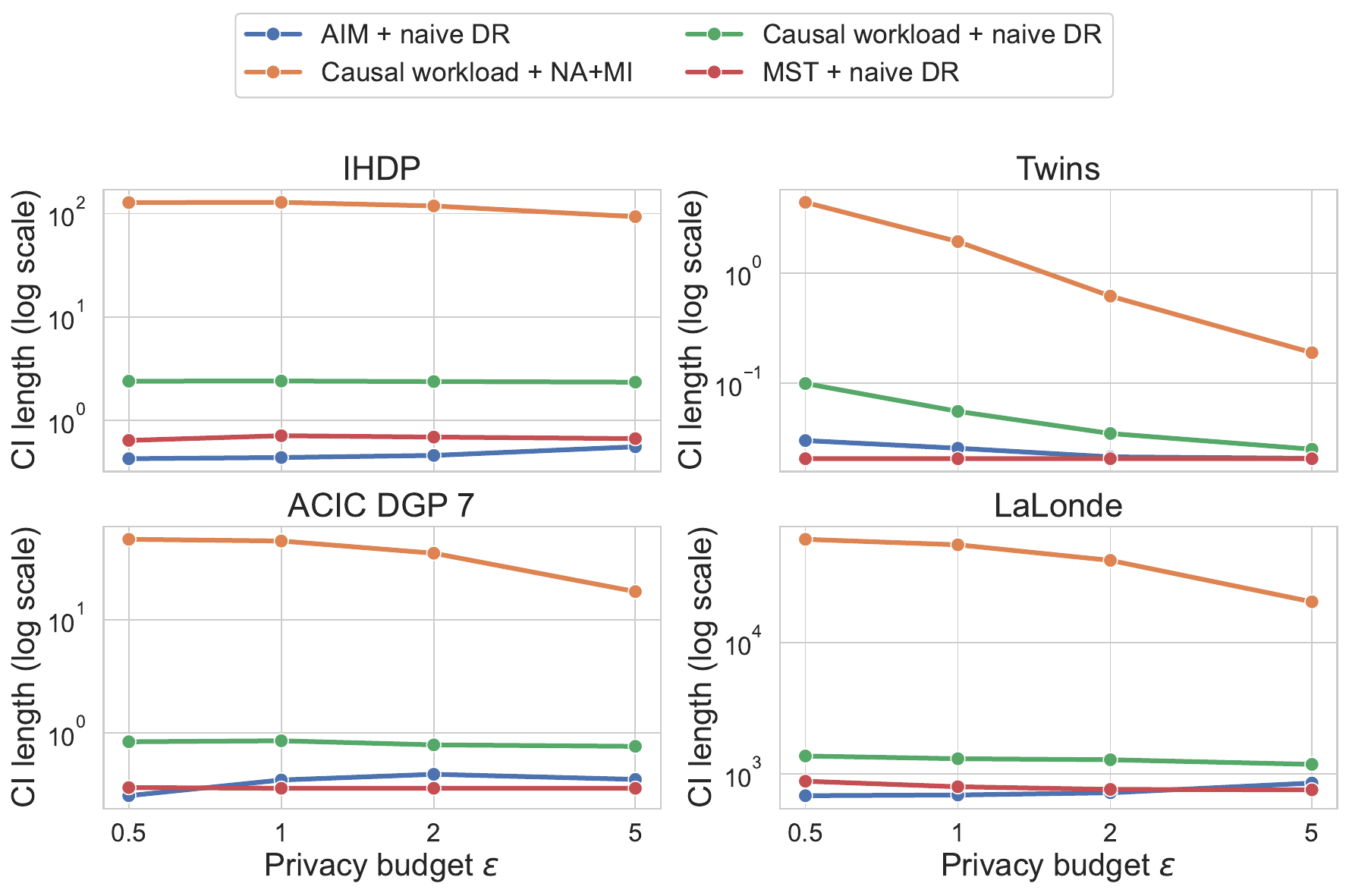}
\caption{Average 95\% CI length across methods, datasets, and privacy budgets ($n_{\mathrm{rep}}=500$).
\cred{Noise-aware MI intervals are $34$--$300\times$ wider than naive intervals at $\eps\le 1$ (depending on dataset and comparator) \dslp[red]{because they account for both DP noise and sampling variability}.
Naive intervals are narrow but far below nominal coverage (Figure~\ref{fig:exp2_coverage}), so their apparent precision is spurious.}}
\label{fig:exp2_length}
\end{figure}

\begin{figure}[!htbp]
\centering
\includegraphics[width=\columnwidth]{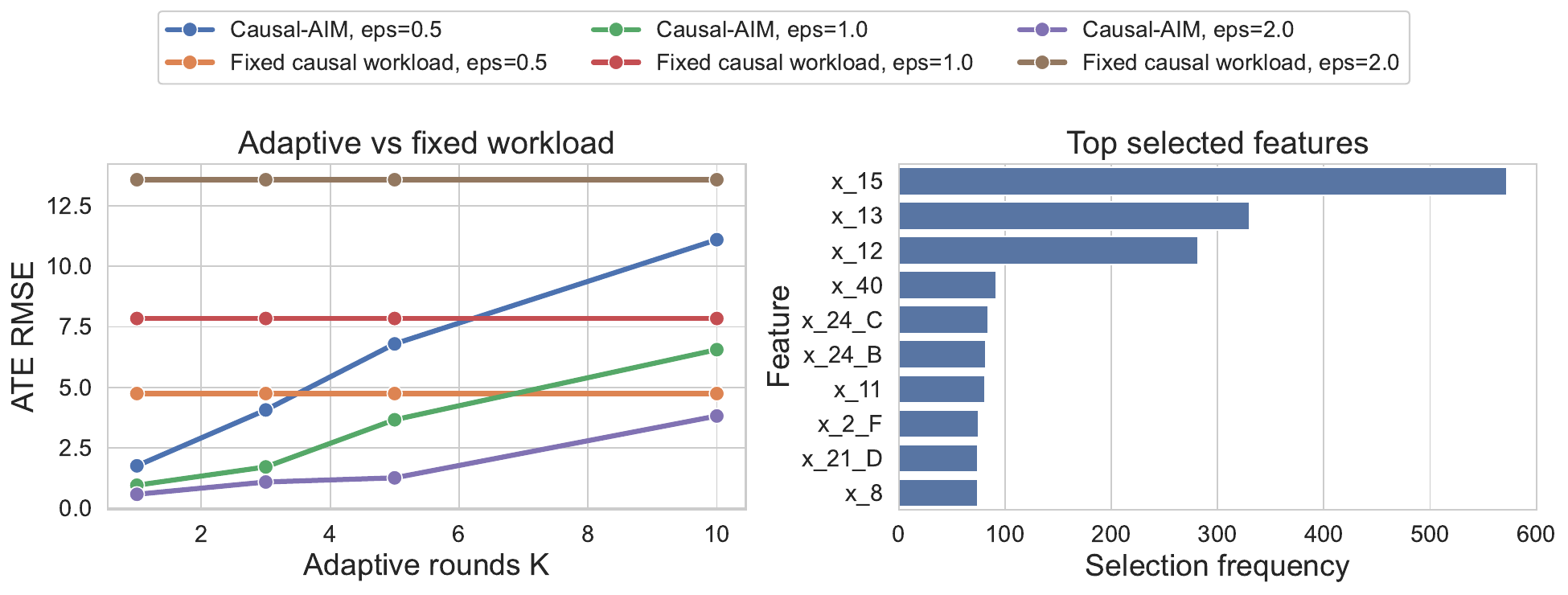}
\caption{ATE RMSE of \textsc{Causal-AIM} (adaptive) versus a fixed causal workload on ACIC DGP~7 ($n_{\mathrm{rep}}=100$).
\cred{The left panel overlays all privacy budgets: horizontal lines mark the fixed workload and curves show \textsc{Causal-AIM} as a function of adaptive rounds $K$; the right panel shows the most frequently selected features.}
\cred{At $\eps\ge 1$, \textsc{Causal-AIM} reduces RMSE by $78{-}96\%$ over the fixed workload at $K\le 3$ ($53{-}91\%$ at $K{=}5$).
At $\eps=0.5$, additional rounds degrade performance because each round consumes budget that exceeds the benefit of better feature targeting; few rounds ($K\le 3$) are recommended at low $\eps$.}}
\label{fig:exp3_adaptive}
\vspace{-.5em}
\end{figure}

\section{Conclusion}\label{sec:discussion}

\codexedit{This paper argues that DP synthetic data for causal inference should preserve the moments used by the causal estimand, not only generic distributional fidelity.}
\codexedit{The proposed causal workload releases treatment-arm feature masses and outcome-feature moments, reconstructs a maximum-entropy synthetic distribution, and supports noise-aware inference through NA+MI.}
\codexedit{The theory tracks how DP moment error, workload approximation, calibration mismatch, and synthetic Monte Carlo error enter ATE estimation; the experiments show the corresponding tradeoff: generic workloads often lead on RMSE at loose budgets, but their naive intervals undercover severely, while causal workload + NA+MI provides calibrated inference at strict privacy budgets.}
\codexedit{The distinction between the direct $q$-route and the synthetic-data route clarifies what the theory controls and what the release enables: one measured workload can be used directly by stable moment maps or converted into reusable synthetic records for ATE, ATT, subgroup effects, and model checks under the same privacy cost.}
\codexedit{The main limitations are the need to choose a feature map $\phi$, approximation bias at high $\eps$, wider intervals when DP noise dominates, and our focus on the AIM/Private-PGM reconstruction family; Appendix~\ref{app:discussion} discusses these limitations, the coverage--privacy diagnostic, and future directions such as continuous feature maps, richer subgroup/CATE workloads, and end-to-end DP-aware feature selection.}

\begin{acknowledgements}
\vspace{-.5em}
\cred{AA was partially supported by NIH grant R01MH139379 and by Patient-Centered Outcomes Research Institute (PCORI) award ME-2023C1-32148.}

\paragraph{Reproducibility.} All experiments, tables, and figures in this paper can be reproduced with the code, data-processing scripts, and notebooks available at \url{https://github.com/AsiaeeLab/causal-aim}.

\end{acknowledgements}

\bibliography{references}

@book{dwork2014foundations,
  title        = {The Algorithmic Foundations of Differential Privacy},
  author       = {Dwork, Cynthia and Roth, Aaron},
  year         = {2014},
  publisher    = {Now Publishers},
  series       = {Foundations and Trends in Theoretical Computer Science},
  volume       = {9},
  pages        = {211--407}
}

@article{mckenna2021nist,
  title        = {Winning the {NIST} Contest: A Scalable and General Approach to Differentially Private Synthetic Data},
  author       = {McKenna, Ryan and Miklau, Gerome and Sheldon, Daniel},
  journal      = {arXiv preprint arXiv:2108.04978},
  year         = {2021}
}

@article{tao2022benchmark,
  title        = {Benchmarking Differentially Private Synthetic Data Generation Algorithms},
  author       = {Tao, Yuchao and McKenna, Ryan and Hay, Michael and Machanavajjhala, Ashwin and Miklau, Gerome},
  journal      = {arXiv preprint arXiv:2112.09238},
  year         = {2021}
}

@article{mckenna2022aim,
  title        = {{AIM}: An Adaptive and Iterative Mechanism for Differentially Private Synthetic Data},
  author       = {McKenna, Ryan and Mullins, Brett and Sheldon, Daniel and Miklau, Gerome},
  journal      = {Proceedings of the VLDB Endowment},
  volume       = {15},
  number       = {11},
  pages        = {2599--2612},
  year         = {2022},
  doi          = {10.14778/3551793.3551817}
}

@inproceedings{donhauser2024privpgd,
  title        = {Privacy-Preserving Data Release Leveraging Optimal Transport and Particle Gradient Descent},
  author       = {Donhauser, Konstantin and Abad, Javier and Hulkund, Neha and Yang, Fanny},
  booktitle    = {Proceedings of the 41st International Conference on Machine Learning (ICML)},
  year         = {2024}
}

@inproceedings{raisa2023noiseaware,
  title        = {Noise-Aware Statistical Inference with Differentially Private Synthetic Data},
  author       = {R{\"a}is{\"a}, Ossi and J{\"a}lk{\"o}, Joonas and Kaski, Samuel and Honkela, Antti},
  booktitle    = {Proceedings of The 26th International Conference on Artificial Intelligence and Statistics (AISTATS)},
  series       = {Proceedings of Machine Learning Research},
  volume       = {206},
  pages        = {3620--3643},
  year         = {2023},
  publisher    = {PMLR}
}

@article{perez2024mwutest,
  title        = {Does Differentially Private Synthetic Data Lead to Synthetic Discoveries?},
  author       = {Perez, Ileana Montoya and Movahedi, Parisa and Nieminen, Valtteri and Airola, Antti and Pahikkala, Tapio},
  journal      = {Methods of Information in Medicine},
  volume       = {63},
  pages        = {35--51},
  year         = {2024},
  doi          = {10.1055/a-2385-1355}
}

@inproceedings{farzam2024causal,
  title        = {Causal Inference under Differential Privacy: Challenges and Mitigation Strategies},
  author       = {Farzam, Amirhossein and Sapiro, Guillermo},
  booktitle    = {Causal Representation Learning Workshop at NeurIPS},
  year         = {2024},
  note         = {CRL@NeurIPS 2024}
}

@article{ohnishi2024covbal,
  title        = {Differentially Private Covariate Balancing Causal Inference},
  author       = {Ohnishi, Yuki and Awan, Jordan},
  journal      = {arXiv preprint arXiv:2410.14789},
  year         = {2024}
}

@article{schroder2025private,
  title        = {{PrivATE}: Differentially Private Confidence Intervals for Average Treatment Effects},
  author       = {Schr{\"o}der, Maresa and Hartenstein, Justin and Feuerriegel, Stefan},
  journal      = {arXiv preprint arXiv:2505.21641},
  year         = {2025}
}

@inproceedings{dp-cate2025,
  title        = {Differentially Private Learners for Heterogeneous Treatment Effects},
  author       = {Schr{\"o}der, Maresa and Melnychuk, Valentyn and Feuerriegel, Stefan},
  booktitle    = {International Conference on Learning Representations (ICLR)},
  year         = {2025}
}

@article{chernozhukov2018dml,
  title        = {Double/Debiased Machine Learning for Treatment and Structural Parameters},
  author       = {Chernozhukov, Victor and Chetverikov, Denis and Demirer, Mert and Duflo, Esther and Hansen, Christian and Newey, Whitney and Robins, James},
  journal      = {The Econometrics Journal},
  volume       = {21},
  number       = {1},
  pages        = {C1--C68},
  year         = {2018}
}

@article{hill2011bayesian,
  title        = {Bayesian Nonparametric Modeling for Causal Inference},
  author       = {Hill, Jennifer L.},
  journal      = {Journal of Computational and Graphical Statistics},
  volume       = {20},
  number       = {1},
  pages        = {217--240},
  year         = {2011}
}

@inproceedings{mckenna2019privatepgm,
  title        = {Graphical-Model Based Estimation and Inference for Differential Privacy},
  author       = {McKenna, Ryan and Sheldon, Daniel and Miklau, Gerome},
  booktitle    = {Proceedings of the 36th International Conference on Machine Learning (ICML)},
  series       = {Proceedings of Machine Learning Research},
  volume       = {97},
  pages        = {4187--4196},
  year         = {2019},
  publisher    = {PMLR}
}

@book{rubin1987mi,
  title        = {Multiple Imputation for Nonresponse in Surveys},
  author       = {Rubin, Donald B.},
  year         = {1987},
  publisher    = {John Wiley \& Sons},
  address      = {New York}
}

@article{rosenbaum1983propensity,
  title        = {The Central Role of the Propensity Score in Observational Studies for Causal Effects},
  author       = {Rosenbaum, Paul R. and Rubin, Donald B.},
  journal      = {Biometrika},
  volume       = {70},
  number       = {1},
  pages        = {41--55},
  year         = {1983}
}

@article{robins1994aipw,
  title        = {Estimation of Regression Coefficients When Some Regressors Are Not Always Observed},
  author       = {Robins, James M. and Rotnitzky, Andrea and Zhao, Lue Ping},
  journal      = {Journal of the American Statistical Association},
  volume       = {89},
  number       = {427},
  pages        = {846--866},
  year         = {1994}
}

@article{dorie2019acic,
  title        = {Automated Versus Do-It-Yourself Methods for Causal Inference: Lessons Learned from a Data Analysis Competition},
  author       = {Dorie, Vincent and Hill, Jennifer and Shalit, Uri and Scott, Marc and Cervone, Dan},
  journal      = {Statistical Science},
  volume       = {34},
  number       = {1},
  pages        = {43--68},
  year         = {2019}
}

@article{lalonde1986evaluating,
  title        = {Evaluating the Econometric Evaluations of Training Programs with Experimental Data},
  author       = {LaLonde, Robert J.},
  journal      = {American Economic Review},
  volume       = {76},
  number       = {4},
  pages        = {604--620},
  year         = {1986}
}

@article{reiter2003inference,
  title        = {Inference for Partially Synthetic, Public Use Microdata Sets},
  author       = {Reiter, Jerome P.},
  journal      = {Survey Methodology},
  volume       = {29},
  number       = {2},
  pages        = {181--188},
  year         = {2003}
}

@article{raghunathan2003multiple,
  title        = {Multiple Imputation for Statistical Disclosure Limitation},
  author       = {Raghunathan, Trivellore E. and Reiter, Jerome P. and Rubin, Donald B.},
  journal      = {Journal of Official Statistics},
  volume       = {19},
  number       = {1},
  pages        = {1--16},
  year         = {2003}
}

@inproceedings{hardt2012mwem,
  title        = {A Simple and Practical Algorithm for Differentially Private Data Release},
  author       = {Hardt, Moritz and Ligett, Katrina and McSherry, Frank},
  booktitle    = {Advances in Neural Information Processing Systems (NeurIPS)},
  volume       = {25},
  year         = {2012}
}

@article{bowen2020comparative,
  title        = {Comparative Study of Differentially Private Data Synthesis Methods},
  author       = {Bowen, Claire McKay and Liu, Fang},
  journal      = {Statistical Science},
  volume       = {35},
  number       = {2},
  pages        = {280--307},
  year         = {2020}
}

@inproceedings{smith2011privacy,
  title        = {Privacy-Preserving Statistical Estimation with Optimal Convergence Rates},
  author       = {Smith, Adam},
  booktitle    = {Proceedings of the 43rd ACM Symposium on Theory of Computing (STOC)},
  pages        = {813--822},
  year         = {2011}
}

@inproceedings{niu2022dp_cate,
  title        = {Differentially Private Estimation of Heterogeneous Causal Effects},
  author       = {Niu, Fengshi and Nori, Harsha and Quistorff, Brian and Caruana, Rich and Ngwe, Donald and Kannan, Aadharsh},
  booktitle    = {Proceedings of the First Conference on Causal Learning and Reasoning (CLeaR)},
  series       = {Proceedings of Machine Learning Research},
  volume       = {177},
  pages        = {618--633},
  year         = {2022},
  publisher    = {PMLR}
}

@article{lebeda2025model_agnostic,
  title        = {Model Agnostic Differentially Private Causal Inference},
  author       = {Lebeda, Christian Janos and Even, Mathieu and Bellet, Aur{\'e}lien and Josse, Julie},
  journal      = {arXiv preprint arXiv:2505.19589},
  year         = {2025}
}

@inproceedings{louizos2017causal,
  title        = {Causal Effect Inference with Deep Latent-Variable Models},
  author       = {Louizos, Christos and Shalit, Uri and Mooij, Joris M. and Sontag, David and Zemel, Richard and Welling, Max},
  booktitle    = {Advances in Neural Information Processing Systems (NeurIPS)},
  volume       = {30},
  year         = {2017}
}

@inproceedings{ding2021retiring,
  title        = {Retiring Adult: New Datasets for Fair Machine Learning},
  author       = {Ding, Frances and Hardt, Moritz and Miller, John and Schmidt, Ludwig},
  booktitle    = {Advances in Neural Information Processing Systems (NeurIPS)},
  volume       = {34},
  pages        = {6478--6490},
  year         = {2021}
}

@article{watson2016approxmodels,
  title        = {Approximate Models and Robust Decisions},
  author       = {Watson, James and Holmes, Chris},
  journal      = {Statistical Science},
  volume       = {31},
  number       = {4},
  pages        = {465--489},
  year         = {2016},
  publisher    = {Institute of Mathematical Statistics},
  doi          = {10.1214/16-STS592},
}

@article{miller2018coarsened,
  title        = {Robust {B}ayesian Inference via Coarsening},
  author       = {Miller, Jeffrey W. and Dunson, David B.},
  journal      = {Journal of the American Statistical Association},
  volume       = {114},
  number       = {527},
  pages        = {1113--1125},
  year         = {2019},
  publisher    = {Taylor \& Francis},
  doi          = {10.1080/01621459.2018.1469995}
}

@article{grunwald2017safebayes,
  title        = {Inconsistency of {B}ayesian Inference for Misspecified Linear Models, and a Proposal for Repairing It},
  author       = {Gr{\"u}nwald, Peter and van Ommen, Thijs},
  journal      = {Bayesian Analysis},
  volume       = {12},
  number       = {4},
  pages        = {1069--1103},
  year         = {2017},
  doi          = {10.1214/17-BA1085},
}

@article{bissiri2016genbayes,
  title        = {A General Framework for Updating Belief Distributions},
  author       = {Bissiri, P. G. and Holmes, C. C. and Walker, S. G.},
  journal      = {Journal of the Royal Statistical Society: Series B (Statistical Methodology)},
  volume       = {78},
  number       = {5},
  pages        = {1103--1130},
  year         = {2016},
  doi          = {10.1111/rssb.12158}
}

@inproceedings{gillenwater2021dpquantiles,
  title        = {Differentially Private Quantiles},
  author       = {Gillenwater, Jennifer and Joseph, Matthew and Kulesza, Alex},
  booktitle    = {Proceedings of the 38th International Conference on Machine Learning (ICML)},
  series       = {Proceedings of Machine Learning Research},
  volume       = {139},
  pages        = {3713--3722},
  year         = {2021},
  publisher    = {PMLR},
}

@inproceedings{kaplan2022dpapproxquantiles,
  title        = {Differentially Private Approximate Quantiles},
  author       = {Kaplan, Haim and Schnapp, Shachar and Stemmer, Uri},
  booktitle    = {Proceedings of the 39th International Conference on Machine Learning (ICML)},
  series       = {Proceedings of Machine Learning Research},
  volume       = {162},
  pages        = {10751--10761},
  year         = {2022},
  publisher    = {PMLR},
}

@inproceedings{lei2011dpmest,
  title        = {Differentially Private {M}-Estimators},
  author       = {Lei, Jing},
  booktitle    = {Advances in Neural Information Processing Systems (NeurIPS)},
  volume       = {24},
  pages        = {361--369},
  year         = {2011},
}

@article{asiaee2026causalwrap,
  title         = {{CausalWrap}: Model-Agnostic Causal Constraint Wrappers for Tabular Synthetic Data},
  author        = {Asiaee, Amir and Liang, Zhuohui J. and Yan, Chao},
  journal       = {arXiv preprint arXiv:2603.02015},
  year          = {2026},
  eprint        = {2603.02015},
  archivePrefix = {arXiv},
  primaryClass  = {cs.LG}
}

@article{jackson2026reward,
  title   = {Reward-Guided Generation Improves the Scientific Utility of Synthetic Biomedical Data},
  author  = {Jackson, Nicholas J. and Espinosa-Dice, Natalia and Yan, Chao and Malin, Bradley A.},
  journal = {medRxiv},
  year    = {2026},
  doi     = {10.64898/2026.03.11.26348077},
  note    = {Preprint; PMCID: PMC13015641}
}

@article{liang_generating_2026,
  title   = {Generating Synthetic Multi-National Longitudinal Cohorts for Clinically Grounded {HIV} Research},
  author  = {Liang, Zhuohui J. and Li, Zhuohang and Jackson, Nicholas J. and Guo, Kevin H. and Caro-Vega, Yanink and Moreira, Ronaldo I. and Paredes, Fabio and Bernadin, Jordany and Varela, Diana and Cesar, Carina and Blasimme, Alessandro and Perkins, Jessica M. and Asiaee, Amir and Duda, Stephany N. and Malin, Bradley A. and Shepherd, Bryan E. and Yan, Chao},
  journal = {Nature Communications},
  year    = {2026},
  doi     = {10.1038/s41467-026-74492-0}
}

\newpage
\onecolumn
\title{Workload-Preserving Differentially Private Synthetic Data\\for Causal Inference via Maximum-Entropy Calibration\\(Supplementary Material)}
\maketitle
\appendix

\section{Feature Encodings and Gram Moments}\label{app:feature-encodings}
\begin{credblock}
\codexedit{This appendix clarifies the feature-encoding point used in Section~\ref{sec:theory}.}
\codexedit{The experiments use \emph{concatenated one-hot} features: each covariate is discretized separately, one-hot encoded separately, and then the groups are stacked.}
\codexedit{For example, with age in $\{\mathrm{young},\mathrm{old}\}$ and sex in $\{\mathrm{female},\mathrm{male}\}$, the experimental-style feature vector is}
\[
\codexedit{
\phi_{\mathrm{cat}}(X)
=
\bigl(1,\ 1\{\mathrm{young}\},\ 1\{\mathrm{old}\},\
1\{\mathrm{female}\},\ 1\{\mathrm{male}\}\bigr).}
\]
\codexedit{The base workload then releases, for each arm, noisy versions of the age-bin masses, sex-bin masses, and the corresponding outcome-weighted moments.}

\codexedit{A \emph{full joint-cell} encoding instead uses indicators for the Cartesian-product cells, e.g.}
\[
\codexedit{
\phi_{\mathrm{joint}}(X)
=
\bigl(1\{\mathrm{young,female}\},\ 1\{\mathrm{young,male}\},\
1\{\mathrm{old,female}\},\ 1\{\mathrm{old,male}\}\bigr).}
\]
\codexedit{Exactly one coordinate of $\phi_{\mathrm{joint}}(X)$ is active for each record.}
\codexedit{Therefore}
\[
\codexedit{
\phi_{\mathrm{joint}}(X)\phi_{\mathrm{joint}}(X)^\top
\text{ is diagonal, and }
G_t(D)=\operatorname{diag}\{q_t^{(0)}(D)\}.}
\]
\codexedit{For concatenated one-hot features, off-diagonal entries are real co-occurrence moments; for instance,}
\[
\codexedit{
\frac1n\sum_i1\{T_i=t\}1\{\mathrm{old}_i\}1\{\mathrm{male}_i\}}
\]
\codexedit{is not determined by the separate old and male counts.}
\codexedit{Thus a direct ridge plug-in under concatenated or overlapping features would need to append $q_t^{(2)}=\operatorname{vec}(G_t)$ as optional queries.}
\codexedit{Our experiments do not use this augmented Gram release; they release the base $4p$ workload, reconstruct synthetic data, and then run DR/AIPW on the synthetic records.}
\end{credblock}

\begin{credblock}
\section{First-Order Sufficiency of the Orthogonal-Score Workload}\label{app:sufficiency}

This section {\color{green}clarifies why} the causal workload of Section~\ref{sec:method} is \emph{one principled choice} rather than a \codexedit{claimed optimum}.
Let $\zeta=(m_0,m_1,e)$ denote the nuisances ($\eta$ remains the positivity constant), $\psi(W;\theta,\zeta)$ a Neyman-orthogonal score for a smooth target $\theta$, and $\mathcal F_\phi$ the linear nuisance class spanned by $\phi$.

\begin{proposition}[First-order sufficiency]
\label{prop:sufficiency}
\codexedit{Restrict the nuisance functions to the feature class spanned by $\phi$.}
\codexedit{For smooth causal targets identified by a Neyman-orthogonal score, the workload $\mathcal Q_\phi=\{q_0^{(0)},q_1^{(0)},q_0^{(1)},q_1^{(1)}\}$ identifies the arm-specific feature masses, outcome-feature moments, and treatment-balance moments used by projected outcome, propensity, IPW, and AIPW moment maps.}
\codexedit{Direct ridge plug-ins may additionally require Gram moments, as discussed in Appendix~\ref{app:feature-encodings}; these are optional extensions rather than part of the default experimental release.}
\end{proposition}

\begin{proof}
\codexedit{The outcome-feature moments $q_t^{(1)}=\E[1\{T=t\}Y\phi(X)]$ encode arm-specific outcome structure in the chosen feature class, while $q_1^{(0)}=\E[T\phi(X)]$ and $q_0^{(0)}=\E[(1-T)\phi(X)]$ encode treatment-feature balance.}
\codexedit{Once projected nuisances or direct cell scores are identified from these moments, orthogonality makes remaining nuisance errors second order, so a synthetic distribution matching the causal workload supports inference for the projected target subject to the explicit privacy, approximation, Monte Carlo, and calibration terms.}
\end{proof}

\paragraph{Necessity intuition and scope.}
If two local alternatives in $\mathcal F_\phi$ agree on a release but differ in an arm-specific outcome-feature or treatment-feature moment, some orthogonal-score target has different first-order pathwise derivatives under the two alternatives, so a release that cannot distinguish those directions cannot support uniformly regular root-$n$ inference over the projected class.
\codexedit{This is a first-order identification argument only: it is \emph{not} a minimax lower bound, does \emph{not} prove uniqueness of the workload, and does \emph{not} show that the base query count is optimal among all DP mechanisms; those questions are outside the present scope (Appendix~\ref{app:discussion}).}
\end{credblock}

\section{Maximum-Entropy Reconstruction Details}\label{app:maxent-details}
\begin{credblock}
\codexedit{This section provides details of the reconstruction step used in Algorithm~\ref{alg:synthesis}.}
\codexedit{Assume first that $W=(X,T,Y)$ has finite support $\mathcal W$, as in a discretized synthetic-data model.}
\codexedit{Let $h_a:\mathcal W\to\R$, $a=1,\ldots,m$, be the ordered workload query functions, with $m=4p$ for the base causal workload and larger $m$ if optional moments are appended.}
\codexedit{For a measured coordinate set $S\subset[m]$, write $h_S(w)=(h_a(w):a\in S)^\top$ and $\widetilde q_S=(\widetilde q_a:a\in S)^\top$.}
\codexedit{Let $p_0(w)>0$ be a base distribution.}
\codexedit{The role of $p_0$ is to say what distribution we prefer among all distributions that match the released moments: it may be uniform over the discretized support, an independent product distribution, a product distribution fit to already measured one-way marginals, or a public prior from an external population.}

\codexedit{If the noisy constraints are feasible exactly, reconstruction is the information projection}
\[
\codexedit{
\min_{p\in\Delta(\mathcal W)}
\sum_{w\in\mathcal W} p(w)\log\frac{p(w)}{p_0(w)}
\quad\text{s.t.}\quad
\sum_{w\in\mathcal W}p(w)h_a(w)=\widetilde q_a,\ a\in S.}
\]
\codexedit{The Lagrangian is}
\[
\codexedit{
\mathcal L(p,\xi,\alpha)
=
\sum_w p(w)\log\frac{p(w)}{p_0(w)}
-\xi^\top\left(\sum_w p(w)h_S(w)-\widetilde q_S\right)
+\alpha\left(\sum_w p(w)-1\right).}
\]
\codexedit{Differentiating with respect to $p(w)$ and setting the derivative to zero gives}
\[
\codexedit{
\log\frac{p(w)}{p_0(w)}+1-\xi^\top h_S(w)+\alpha=0,}
\]
\codexedit{and therefore}
\[
\codexedit{
p_\xi(w)
=
\frac{p_0(w)\exp\{\xi^\top h_S(w)\}}
{Z(\xi)},\qquad
Z(\xi)=\sum_{u\in\mathcal W}p_0(u)\exp\{\xi^\top h_S(u)\}.}
\]
\codexedit{The dual objective can be written as}
\[
\codexedit{
\max_{\xi\in\R^{|S|}}
\left\{\xi^\top\widetilde q_S
-
\log Z(\xi)\right\},}
\]
\codexedit{whose gradient is $\widetilde q_S-\E_{p_\xi}[h_S(W)]$ and whose Hessian is minus the covariance of $h_S(W)$ under $p_\xi$.}
\codexedit{Thus optimizing the dual searches for multipliers $\xi$ whose exponential-family distribution matches the released noisy moments.}

\codexedit{Because $\widetilde q_S$ contains DP noise, exact feasibility is not guaranteed.}
\codexedit{For example, noise can make a released cell mass slightly negative, or make several marginal constraints mutually inconsistent.}
\codexedit{Implementations therefore solve a relaxed weighted calibration problem such as}
\[
\codexedit{
\min_{P\in\mathcal P}
\frac{1}{2}\norm{\Sigma_S^{-1/2}\Pi_S\{q(P)-\widetilde q\}}_2^2
+ \rho\,\KL(P\|P_0),}
\]
\codexedit{where $\Sigma_S$ is the known DP noise covariance on the measured coordinates, so high-noise coordinates are automatically downweighted.}
\codexedit{Equivalently, one can solve the exact dual only on retained high-signal coordinates and stop when residuals are at the noise scale, e.g. $|q_a(P)-\widetilde q_a|\le c_{\mathrm{cal}}\sigma_a$.}
\codexedit{This is the sense in which Algorithm~\ref{alg:synthesis} uses Private-PGM-style reconstruction: it fits a structured exponential-family / graphical-model distribution to noisy measured moments rather than trying to estimate the full joint table without structure, then samples synthetic records from the fitted distribution.}

\begin{algorithm}[h]
\caption{Fixed Causal-Workload Synthesis Procedure}
\label{alg:synthesis}
\begin{algorithmic}[1]
\Require \codexedit{Confidential data $D$, query functions $\{h_a\}_{a=1}^{m}$, privacy budget $(\eps,\del)$, base distribution $P_0$, synthetic size $n_{\mathrm{syn}}$}
\State \codexedit{Compute exact empirical answers $q_a(D)=n^{-1}\sum_i h_a(W_i)$ for $a=1,\ldots,m$}
\State \codexedit{Compute the sensitivity bound $\Delta_2$ and Gaussian scale $\sigma=\Delta_2\sqrt{2\log(1.25/\del)}/\eps$}
\State \codexedit{Release $\widetilde q=q(D)+Z$, where $Z\sim\N(0,\sigma^2 I_m)$}
\State \codexedit{Optionally retain high-signal coordinates $S_\tau=\{a:|\widetilde q_a|/\sigma_a\ge\tau_{\mathrm{SNR}}\}$; otherwise set $S_\tau=[m]$}
\State \codexedit{Fit $P^{\mathrm{syn}}$ by solving the relaxed maximum-entropy calibration problem on $\Pi_{S_\tau}\widetilde q$}
\State \codexedit{Draw $n_{\mathrm{syn}}$ records $D^{\mathrm{syn}}$ from $P^{\mathrm{syn}}$}
\Ensure \codexedit{Synthetic dataset $D^{\mathrm{syn}}$ and recorded noise scale $\sigma$}
\end{algorithmic}
\end{algorithm}
\end{credblock}

\section{\textsc{Causal-AIM} Selection Details}\label{app:causal-aim}

\codexedit{At each round $k$, \textsc{Causal-AIM} computes an internal utility for each unmeasured candidate group $r$, estimating the reduction in ATE error if its coordinate block $I(r)\subset[m]$ were added to the measured workload.}
\codexedit{This is an adaptive greedy selection rule with private randomized selection: candidates are scored against the current reconstructed model, one group is selected through the exponential mechanism, its coordinates are measured once, and the model is refit before the next round.}
\codexedit{Previously measured groups are removed from the candidate set and are not re-measured; later rounds reuse their stored noisy answers.}

\codexedit{The score is motivated by the orthogonal-score residual}
\[
\codexedit{\Delta_r = \norm{\E_{P_{k-1}}[\psi(W; \hat\theta_{k-1}, \hat\zeta_{k-1}) \cdot \phi_r(X)]}_2^2,}
\]
\codexedit{where $\psi$ is the influence function for the ATE and $(\hat\theta_{k-1}, \hat\zeta_{k-1})$ are the current nuisance estimates under $P_{k-1}$.}
\codexedit{Here $W=(X,T,Y)$; for the ATE, with nuisances $\zeta=(m_0,m_1,e)$ and target $\tau$, a standard orthogonal score is}
\[
\codexedit{
\psi(W;\tau,\zeta)
=
m_1(X)-m_0(X)
+\frac{T\{Y-m_1(X)\}}{e(X)}
-\frac{(1-T)\{Y-m_0(X)\}}{1-e(X)}
-\tau.}
\]
\codexedit{In the implementation, this is operationalized by a moment-space proxy.}
\codexedit{Let $q^{(k)}$ be the current noisy moment vector assembled from the coordinates measured so far and from the public base distribution on unmeasured coordinates.}
\codexedit{A typical utility is}
\[
\codexedit{
u_k(r)=\widehat L_k\,\bigl\lVert \Pi_{I(r)}\{q(D)-q^{(k-1)}\}\bigr\rVert_2,}
\]
\codexedit{where $\widehat L_k$ is an empirical Lipschitz/overlap factor and $\Pi_{I(r)}$ projects onto the coordinates for group $r$.}
\codexedit{Large utility means that the current reconstruction leaves a causally weighted moment block poorly matched.}
\codexedit{The utilities are used only inside the selection mechanism and are not released.}
\codexedit{If $\Delta_u$ bounds the sensitivity of $u_k$, the exponential mechanism selects}
\[
\codexedit{
\Pr(r_k=r)
\propto
\exp\!\left\{\frac{\eps_{\mathrm{score},k}\,u_k(r)}{2\Delta_u}\right\},
\qquad r\in A_{k-1}.}
\]
\codexedit{Thus larger utility means larger selection probability; if one instead defines a loss, the utility is the negative loss.}

\begin{algorithm}[h]
\caption{\textsc{Causal-AIM} Selection and Measurement}
\label{alg:causalaim-detail}
\begin{algorithmic}[1]
\Require \codexedit{Data $D$, groups $\mathcal R$, coordinate blocks $\{I(r)\}_{r\in\mathcal R}$, budget $(\eps,\del)$, rounds $K$, public base moment vector $q^{(0)}=q(P_0)$}
\State \codexedit{Split the budget into scoring and measurement portions, and split each portion across rounds by sequential composition}
\State \codexedit{Initialize measured coordinates $S_0\leftarrow\emptyset$, available groups $A_0\leftarrow\mathcal R$, and current moment proxy $q^{(0)}$}
\For{$k=1,\ldots,\min(K,|\mathcal R|)$}
    \State \codexedit{For each $r\in A_{k-1}$, compute an internal utility $u_k(r)$ for adding coordinate block $I(r)$}
    \State \codexedit{Select $r_k$ from $A_{k-1}$ with the exponential mechanism, with probability proportional to $\exp\{\eps_{\mathrm{score},k}u_k(r)/(2\Delta_u)\}$}
    \State \codexedit{Release noisy answers $\widetilde q_{I(r_k)}=q_{I(r_k)}(D)+Z_k$ with the Gaussian mechanism}
    \State \codexedit{Update $S_k\leftarrow S_{k-1}\cup I(r_k)$ and $A_k\leftarrow A_{k-1}\setminus\{r_k\}$}
    \State \codexedit{Update $q^{(k)}$ by replacing coordinates $I(r_k)$ with $\widetilde q_{I(r_k)}$ and retaining all earlier noisy measurements}
    \State \codexedit{Refit $P_k$ by maximum-entropy calibration to the stored noisy answers on $S_k$}
\EndFor
\State \codexedit{Output stored noisy moments $\widetilde q_{S_K}$; for synthetic release, output $P_K$ and sample synthetic records from $P_K$}
\end{algorithmic}
\end{algorithm}

\section{Proofs}\label{app:proofs}

\subsection{Weighted ATEs: ATT and Subgroup Effects}\label{app:weighted-ate}
\begin{credblock}
\codexedit{The ATE decomposition extends to weighted effects whenever the weighting rule is represented by the retained workload coordinates.}
For a nonnegative weight function $\omega$, define
\[
\codexedit{
\tau_\omega
:=
\frac{\E[\omega(X)\{m_1(X)-m_0(X)\}]}
{\E[\omega(X)]}.}
\]
\codexedit{This includes subgroup effects with $\omega(X)=\ind\{X\in G\}$, and ATT-type effects with $\omega(X)=e(X)$, provided the subgroup or propensity weights are represented by the workload.}

\begin{proposition}[Weighted-effect moment stability]
\label{prop:weighted-ate}
\codexedit{Let $\widehat\tau_\omega(q)=N_\omega(q)/D_\omega(q)$ be a direct moment-map estimator using only retained workload coordinates.}
\codexedit{Assume that, on the admissible workload domain, $D_\omega(q)\ge\rho_\omega>0$, $|N_\omega(q)|\le M_\omega D_\omega(q)$, and}
\[
\codexedit{
|N_\omega(q)-N_\omega(q')|\le L_{N,\omega}\norm{q-q'}_\infty,\qquad
|D_\omega(q)-D_\omega(q')|\le L_{D,\omega}\norm{q-q'}_\infty.}
\]
\codexedit{Then $\widehat\tau_\omega$ is Lipschitz in the workload moments:}
\[
\codexedit{
|\widehat\tau_\omega(q)-\widehat\tau_\omega(q')|
\le
L_\omega\norm{q-q'}_\infty,\qquad
L_\omega
=
\frac{L_{N,\omega}+M_\omega L_{D,\omega}}{\rho_\omega}.}
\]
\codexedit{Consequently, the proof of Theorem~\ref{thm:ate_bound} applies to $\tau_\omega$ after replacing $L_{\mathrm{est}}$ by $L_\omega$ and replacing $\mathrm{Approx}(\phi;S)$ by the weighted approximation error $\mathrm{Approx}_\omega(\phi;S)$.}
\end{proposition}

\begin{proof}
\codexedit{Let $\delta_q=\norm{q-q'}_\infty$.}
\codexedit{By adding and subtracting $N_\omega(q')/D_\omega(q)$,}
\[
\codexedit{
\left|
\frac{N_\omega(q)}{D_\omega(q)}
-
\frac{N_\omega(q')}{D_\omega(q')}
\right|
\le
\frac{|N_\omega(q)-N_\omega(q')|}{D_\omega(q)}
+
|N_\omega(q')|
\left|
\frac{1}{D_\omega(q)}
-
\frac{1}{D_\omega(q')}
\right|.}
\]
\codexedit{The first term is at most $L_{N,\omega}\delta_q/\rho_\omega$.}
\codexedit{For the second term, use $|N_\omega(q')|\le M_\omega D_\omega(q')$ and $D_\omega(q)\ge\rho_\omega$ to obtain}
\[
\codexedit{
|N_\omega(q')|
\frac{|D_\omega(q)-D_\omega(q')|}
{D_\omega(q)D_\omega(q')}
\le
\frac{M_\omega L_{D,\omega}}{\rho_\omega}\delta_q.}
\]
\codexedit{Combining the two displays gives the Lipschitz constant.}
\codexedit{The decomposition proof of Theorem~\ref{thm:ate_bound} then applies verbatim with target $\tau_\omega$ and represented target $\tau_{\omega,\phi,S}$.}
\end{proof}

\codexedit{For a subgroup effect, $\omega(X)=\ind\{X\in G\}$ is represented whenever $G$ is a union of workload cells; the denominator condition is a minimum subgroup-mass condition.}
\codexedit{For ATT in a partition basis, $\omega(X)=e(X)$ is represented by the cell propensity $e_j=p_{1j}/(p_{0j}+p_{1j})$, computed from $q^{(0)}$ and clipped as in Appendix~\ref{app:ipw-aipw-stability}; the same cell-mass and clipping assumptions give the required Lipschitz constants.}
\codexedit{This proposition does not claim a full pointwise CATE guarantee; pointwise CATE requires additional approximation assumptions at the target covariate value.}
\end{credblock}

\subsection{Proof of Theorem~\ref{thm:lipschitz}: ATE Lipschitzness}\label{app:proof-lipschitz}
\begin{credblock}
\codexedit{Throughout, $G_t=n^{-1}\sum_i \ind\{T_i=t\}\phi(X_i)\phi(X_i)^\top$ (equivalently $q_t^{(2)}=\operatorname{vec}(G_t)$ when released explicitly) and $r_t=q_t^{(1)}$ are the arm-specific Gram and outcome moments, $\bar\phi:=q_0^{(0)}+q_1^{(0)}$ is the empirical feature mean, and $u=(G_0,G_1,r_0,r_1,\bar\phi)$ collects the moment inputs.}
\codexedit{The ATE plug-in is $\widehat\tau_\lambda(u)=\bar\phi^\top(\widehat\beta_{1,\lambda}-\widehat\beta_{0,\lambda})$ with $\widehat\beta_{t,\lambda}=(G_t+\lambda I)^{-1}r_t$.}

\begin{lemma}[Ridge perturbation bound, full form]\label{lem:ridge-full}
Under the assumptions of Theorem~\ref{thm:lipschitz}, consider two inputs $u,u'$ with $G'_t+\lambda I\succeq(\kappa+\lambda)I/2$ and $\norm{\widehat\beta_{t,\lambda}(u)}_2\vee\norm{\widehat\beta_{t,\lambda}(u')}_2\le R_\beta$ for $t=0,1$. Then
\begin{align}
|\widehat\tau_\lambda(u)-\widehat\tau_\lambda(u')|
&\le 2R_\beta\norm{\bar\phi-\bar\phi'}_2 \notag\\
&\hspace{-4em}+ \frac{2\phi_{\max}}{\kappa+\lambda}\sum_{t=0}^1\left\{\norm{r_t-r'_t}_2+R_\beta\norm{G_t-G'_t}_{\mathrm{op}}\right\}.
\label{eq:ridge-lipschitz-full}
\end{align}
\end{lemma}

\begin{proof}
Let $A_t=G_t+\lambda I$ and $A'_t=G'_t+\lambda I$, so $\norm{A_t^{-1}}_{\mathrm{op}}\le 1/(\kappa+\lambda)$ and $\norm{(A'_t)^{-1}}_{\mathrm{op}}\le 2/(\kappa+\lambda)$.
The identity
$A_t^{-1}r_t-(A'_t)^{-1}r'_t = A_t^{-1}(r_t-r'_t) + A_t^{-1}(G'_t-G_t)(A'_t)^{-1}r'_t$
gives
\[
\norm{\widehat\beta_{t,\lambda}(u)-\widehat\beta_{t,\lambda}(u')}_2
\le \frac{\norm{r_t-r'_t}_2+R_\beta\norm{G_t-G'_t}_{\mathrm{op}}}{\kappa+\lambda},
\]
with constants at most doubled when only the lower bound on $A'_t$ is used.
Splitting the plug-in difference,
\begin{align*}
&|\bar\phi^\top(\widehat\beta_{1,\lambda}(u)-\widehat\beta_{0,\lambda}(u))
-(\bar\phi')^\top(\widehat\beta_{1,\lambda}(u')-\widehat\beta_{0,\lambda}(u'))|\\
&\le \norm{\bar\phi-\bar\phi'}_2\,\norm{\widehat\beta_{1,\lambda}(u)-\widehat\beta_{0,\lambda}(u)}_2\\
&\quad+ \norm{\bar\phi'}_2\sum_{t=0}^1\norm{\widehat\beta_{t,\lambda}(u)-\widehat\beta_{t,\lambda}(u')}_2 ,
\end{align*}
where $\norm{\widehat\beta_{1,\lambda}(u)-\widehat\beta_{0,\lambda}(u)}_2\le 2R_\beta$ by the radius assumption and $\norm{\bar\phi'}_2\le\phi_{\max}$ since $\bar\phi'$ averages feature vectors bounded by $\phi_{\max}$.
Substituting the coefficient bound proves \eqref{eq:ridge-lipschitz-full}.
\end{proof}

\emph{Proof of Theorem~\ref{thm:lipschitz}.}
\codexedit{Let $\delta_q=\norm{q-q'}_\infty$ denote the maximum coordinatewise perturbation across the moment blocks entering $u=(G_0,G_1,r_0,r_1,\bar\phi)$.}
\codexedit{Then $\norm{\bar\phi-\bar\phi'}_2\le\sqrt p\,\delta_q$, $\norm{r_t-r_t'}_2\le\sqrt p\,\delta_q$, and $\norm{G_t-G_t'}_{\mathrm{op}}\le p\,\delta_q$.}
\codexedit{Substituting these inequalities into Lemma~\ref{lem:ridge-full} gives}
\[
\codexedit{
|\widehat\tau_\lambda(q)-\widehat\tau_\lambda(q')|
\le
\left\{
2R_\beta\sqrt p
+\frac{4\phi_{\max}(\sqrt p+R_\beta p)}{\kappa+\lambda}
\right\}\delta_q.}
\]
\codexedit{Thus one valid main-text choice is}
\[
\codexedit{
C_\phi
=
2R_\beta\sqrt p+4\phi_{\max}(\sqrt p+R_\beta p),}
\]
\codexedit{which yields $L_\phi\le C_\phi\{1+(\kappa+\lambda)^{-1}\}$.}
\codexedit{For propensity-clipped IPW and orthogonal-score estimators, clipping $\widehat e(x)$ to $[\eta,1-\eta]$ makes $a\mapsto 1/a$ and $a\mapsto 1/(1-a)$ Lipschitz with constants of order $\eta^{-2}$ on the clipped interval, and the same decomposition applies with additional polynomial factors in $1/\eta$.}
\codexedit{The unregularized case follows by setting $\lambda=0$ when $\kappa>0$.}
\qed

\begin{remark}[Why the coefficient radius is explicit]
The $O(1/(\kappa+\lambda))$ dependence for Gram perturbations uses $\norm{\widehat\beta_{t,\lambda}}_2\le R_\beta$, enforceable by coefficient clipping or a bounded parameter set (both post-processing).
Without a radius, the same identity yields a valid bound with the Gram term scaling as $O(\norm{r_t}_2/(\kappa+\lambda)^2)$.
We state the radius explicitly so the theorem does not hide a regularity condition inside the constant.
\end{remark}

\emph{Proof of Corollary~\ref{cor:noise_ridge}.}
\codexedit{By Gaussian tails and a union bound, with probability at least $1-\gamma$, the released workload vector obeys $\norm{\widetilde q-q(D)}_\infty\le \delta_m=c_0\sigma\sqrt{\log(m/\gamma)}$.}
\codexedit{In a full joint-cell basis the Gram perturbation is diagonal, so $\max_t\norm{\widetilde G_t-G_t}_{\mathrm{op}}\le\delta_m$.}
\codexedit{For explicitly released dense Gram blocks, the same conclusion follows under the stated operator-norm event; a standard Gaussian-matrix bound gives this event when $\lambda$ is at least a constant multiple of $\sigma\sqrt{p+\log(1/\gamma)}$.}
\[
\codexedit{\lambda_{\min}(\widetilde G_t+\lambda I)\ge\kappa+\lambda-\norm{\widetilde G_t-G_t}_{\mathrm{op}}\ge\kappa+\lambda/2}
\]
\codexedit{whenever $\max_t\norm{\widetilde G_t-G_t}_{\mathrm{op}}\le\lambda/2$.}
\codexedit{Theorem~\ref{thm:lipschitz} then gives the displayed main-text rate.}
\qed

\paragraph{Conditioning intuition.}\label{app:conditioning-ridge}
\codexedit{With full joint-cell indicators, an arm-specific Gram matrix loses rank exactly when some cell contains few or no records from that treatment arm.}
\codexedit{With concatenated bin indicators or other overlapping bases, the same issue appears through missing co-occurrences or near-collinear Gram blocks.}
\codexedit{DP noise of scale $\sigma$ further perturbs each released coordinate, so at small $\eps$ and large $p$ the effective $\kappa$ can be zero and the unregularized moment-to-ATE map is not Lipschitz.}
\codexedit{Corollary~\ref{cor:noise_ridge} restores stability by setting $\lambda$ at the privacy-noise scale.}
\codexedit{This plays a different role from SNR thresholding: SNR thresholding stops the max-entropy solver from chasing moments that are indistinguishable from privacy noise, while ridge keeps the nuisance inverse stable when the retained moments still leave an arm's Gram matrix near-singular.}
\end{credblock}

\subsection{Stability of Clipped IPW and AIPW Scores}\label{app:ipw-aipw-stability}
\begin{credblock}
\codexedit{This appendix makes Proposition~\ref{prop:aipw-stability} explicit for partition-feature workloads.}
\codexedit{The argument is standard perturbation calculus for clipped inverse-propensity scores; it is included only to show how the same moment-error bound propagates through IPW/AIPW post-processing.}

\begin{lemma}[Clipped IPW/AIPW perturbation]
\codexedit{Let $\phi_j(x)=\ind\{x\in A_j\}$, $j=1,\ldots,p$, be partition indicators, $|Y|\le B$, and define $p_{tj}(q)=q_{tj}^{(0)}$, $s_j(q)=p_{0j}(q)+p_{1j}(q)$, and $r_{tj}(q)=q_{tj}^{(1)}$.}
\codexedit{For two admissible moment vectors $q,q'$, set $\delta_q=\norm{q-q'}_\infty$ and $\Delta_e=\max_j|\widehat e_j(q)-\widehat e_j(q')|$, where $\widehat e_j(\cdot)\in[\eta,1-\eta]$.}
\codexedit{Then the clipped IPW functional}
\[
\codexedit{
\widehat\tau_{\mathrm{IPW}}(q)
=
\sum_{j=1}^p
\left\{
\frac{r_{1j}(q)}{\widehat e_j(q)}
-
\frac{r_{0j}(q)}{1-\widehat e_j(q)}
\right\}}
\]
\codexedit{satisfies}
\[
\codexedit{
|\widehat\tau_{\mathrm{IPW}}(q)-\widehat\tau_{\mathrm{IPW}}(q')|
\le
\frac{2p}{\eta}\delta_q
+
\frac{B}{\eta^2}\Delta_e.}
\]
\codexedit{If $\widehat e_j(q)=\operatorname{clip}\{p_{1j}(q)/s_j(q),\eta,1-\eta\}$ and $\min_{j,q_\circ\in\{q,q'\}}s_j(q_\circ)\ge\rho$, then $\Delta_e\le 3\delta_q/\rho$.}

\codexedit{Now suppose $|\widehat m_{tj}(q)|\le B_m$ and set $\Delta_m=\max_{t,j}|\widehat m_{tj}(q)-\widehat m_{tj}(q')|$.}
\codexedit{The partition AIPW functional}
\[
\codexedit{
\begin{aligned}
\widehat\tau_{\mathrm{AIPW}}(q)
&=
\sum_{j=1}^p s_j(q)\{\widehat m_{1j}(q)-\widehat m_{0j}(q)\}\\
&\quad+
\sum_{j=1}^p\frac{r_{1j}(q)-p_{1j}(q)\widehat m_{1j}(q)}{\widehat e_j(q)}
-
\sum_{j=1}^p\frac{r_{0j}(q)-p_{0j}(q)\widehat m_{0j}(q)}{1-\widehat e_j(q)}
\end{aligned}}
\]
\codexedit{satisfies}
\[
\codexedit{
|\widehat\tau_{\mathrm{AIPW}}(q)-\widehat\tau_{\mathrm{AIPW}}(q')|
\le
\left(4B_m p+\frac{2p(1+B_m)}{\eta}\right)\delta_q
+
2\left(1+\frac{1}{\eta}\right)\Delta_m
+
\frac{B+B_m}{\eta^2}\Delta_e.}
\]
\codexedit{Consequently, if the fitted nuisance maps obey $\Delta_e\le L_e\delta_q$ and $\Delta_m\le L_m\delta_q$, clipped IPW and AIPW are Lipschitz in the released moment vector with the displayed estimator-specific constants.}
\end{lemma}

\codexedit{\emph{Proof.}
For $u,v\in[\eta,1-\eta]$, $|1/u-1/v|\le |u-v|/\eta^2$ and $|1/(1-u)-1/(1-v)|\le |u-v|/\eta^2$.}
\codexedit{Apply these inequalities after adding and subtracting terms with either the moments held fixed or the nuisance functions held fixed.}
\codexedit{The partition basis gives $\sum_j |r_{1j}|+\sum_j|r_{0j}|\le B$, $\sum_j p_{1j}+\sum_j p_{0j}\le1$, and $\sum_j|s_j(q)-s_j(q')|\le2p\delta_q$.}
\codexedit{These bounds yield the IPW and AIPW displays.}
\codexedit{For the cell propensity ratio, clipping is nonexpansive and the map $p_{1j}/s_j$ has perturbation at most $3\delta_q/\rho$ when $s_j\ge\rho$.}
\qed
\end{credblock}

\subsection{Proof of Theorem~\ref{thm:dp_moment}: DP Moment Accuracy}\label{app:proof-dp}

\codexedit{Write $q(D)=n^{-1}\sum_i h(W_i)$ with $\norm{h(W_i)}_2\le H_q$.}
\codexedit{Under replacement adjacency, changing one record changes $q(D)$ by at most}
\[
\codexedit{\Delta_2 \le \frac{2H_q}{n}.}
\]

\codexedit{Applying the Gaussian mechanism to $q(D)$ adds $Z\sim\N(0,\sigma^2 I_m)$ with
$\sigma = \Delta_2 \sqrt{2\log(1.25/\del)}/\eps$.}
\codexedit{By a standard Gaussian tail bound and a union bound over the $m$ coordinates, with probability at least $1-\gamma$,}
\[
\codexedit{\norm{Z}_\infty \le \sigma\sqrt{2\log(m/\gamma)}.}
\]
\codexedit{Substituting the expression for $\sigma$ yields the stated rate.}
\codexedit{For the base workload, $\norm{h(W)}_2^2=(1+Y^2)\norm{\phi(X)}_2^2\le\phi_{\max}^2(1+B^2)$.}
\codexedit{For an optional Gram-augmented workload, one additional active Gram block contributes $\norm{\phi(X)\phi(X)^\top}_F^2=\norm{\phi(X)}_2^4\le\phi_{\max}^4$, giving $H_q^2\le\phi_{\max}^2(1+B^2)+\phi_{\max}^4$.}
In the experiments, we use $B=5$ (outcomes clipped to $[-5,5]$).
\qed

\subsection{Proof of Theorem~\ref{thm:ate_bound}: ATE Error Decomposition}\label{app:proof-ate}

\codexedit{Let $P^{\mathrm{syn}}$ be the actual solver output.}
\codexedit{Write $\widehat\tau(q)$ for the population moment-map estimator and $\widehat\tau^{\mathrm{syn}}$ for its empirical version computed from $n_{\mathrm{syn}}$ synthetic draws.}
\codexedit{Decompose the error via the triangle inequality:}
\begin{align}
|\widehat{\tau}^{\mathrm{syn}} - \tau|
&\le \underbrace{|\tau - \tau_{\phi,S}|}_{\text{(iii) approximation}}
+ \underbrace{|\tau_{\phi,S} - \widehat\tau(q^\star)|}_{\text{= 0 by definition}}
+ \underbrace{|\widehat\tau(q^\star) - \widehat\tau(q(D))|}_{\text{(i) sampling}} \notag \\
&\quad + \underbrace{|\widehat\tau(q(D)) - \widehat\tau(\widetilde q)|}_{\text{(ii) privacy}}
+ \underbrace{|\widehat\tau(\widetilde q)-\widehat\tau(q(P^{\mathrm{syn}}))|}_{\text{(v) calibration}}
+ \underbrace{|\widehat\tau(q(P^{\mathrm{syn}}))-\widehat\tau^{\mathrm{syn}}|}_{\text{(iv) Monte Carlo}},
\end{align}
\codexedit{where $\tau_{\phi,S}$ is the target represented by the retained workload coordinates.}

\emph{Term (i):} By the CLT, $\norm{q(D)-q^\star}_\infty = O_p(n^{-1/2})$ (concentration of bounded averages).
\codexedit{By the estimator Lipschitz condition, $|\widehat\tau(q^\star)-\widehat\tau(q(D))| \le L_{\mathrm{est}}O_p(n^{-1/2})$.}

\emph{Term (ii):} \codexedit{By Theorem~\ref{thm:dp_moment}, $\norm{\widetilde q - q(D)}_\infty = O_p(H_q\sqrt{\log m \cdot \log(1/\del)}/(n\eps))$.}
\codexedit{Applying the same Lipschitz condition gives the privacy term.}

\emph{Term (iii):} This is $\mathrm{Approx}(\phi;S) := |\tau - \tau_{\phi,S}|$, the bias from representing outcome and propensity models in a finite retained basis.
This term is zero when $m_t$ and $e$ are linear in $\phi$ and nonzero otherwise; its magnitude depends on the richness of $\phi$.

\emph{Term (iv):} The synthetic data estimator $\widehat\tau^{\mathrm{syn}}$ is computed from $n_{\mathrm{syn}}$ i.i.d.\ draws from $P^{\mathrm{syn}}$.
By the CLT for the synthetic population, this contributes $O_p(n_{\mathrm{syn}}^{-1/2})$.

\begin{credblock}
\emph{Term (v), calibration:}
\codexedit{If exact matching is infeasible because the measurements are noisy, define the relaxation residual directly as $\Delta_{\mathrm{cal}}(S):=\Pi_S\{\widetilde q-q(P^{\mathrm{syn}})\}$.}
\codexedit{The estimator Lipschitz condition bounds the calibration term by}
\[
\codexedit{
|\widehat\tau(\widetilde q)-\widehat\tau(q(P^{\mathrm{syn}}))|
\le L_{\mathrm{est}}\norm{\Delta_{\mathrm{cal}}(S)}_2
= \mathrm{CalGap}(S,\sigma).}
\]

Combining all five terms yields the stated bound.
\qed

\begin{corollary}[SNR thresholding controls the calibration gap]
\label{cor:snr_calgap}
\codexedit{Let $S_\tau=\{a:\ |\widetilde q_a|/\sigma_a\ge\tau_{\mathrm{SNR}}\}$ be the retained set, where $\sigma_a$ is the Gaussian-mechanism standard deviation of coordinate $a$.}
Suppose the max-entropy solver is run only on $S_\tau$ and stopped at noise-level tolerance
\begin{equation}
\codexedit{|\widetilde q_a-q_a(P^{\mathrm{syn}})| \le c_{\mathrm{cal}}\,\sigma_a
\qquad\text{for all } a\in S_\tau,}
\label{eq:cal-tolerance}
\end{equation}
for a fixed numerical constant $c_{\mathrm{cal}}$.
\codexedit{If $\sigma_a\le\bar\sigma$ on $S_\tau$, then}
$\codexedit{\mathrm{CalGap}(S_\tau,\sigma)\le L_{\mathrm{est}} c_{\mathrm{cal}}\bar\sigma\sqrt{m_{\mathrm{kept}}}}$ with $\codexedit{m_{\mathrm{kept}}}=|S_\tau|$; in the homoskedastic case $\bar\sigma=\sigma$, so $\codexedit{\mathrm{CalGap}(S_\tau,\sigma)=O(L_{\mathrm{est}}\sigma\sqrt{m_{\mathrm{kept}}})}$.
\end{corollary}

\begin{proof}
By definition of $\mathrm{CalGap}$ and the tolerance \eqref{eq:cal-tolerance},
\[
\mathrm{CalGap}(S_\tau,\sigma)
\codexedit{= L_{\mathrm{est}}\Bigl(\sum_{a\in S_\tau}|\widetilde q_a-q_a(P^{\mathrm{syn}})|^2\Bigr)^{1/2}
\le L_{\mathrm{est}} c_{\mathrm{cal}}\Bigl(\sum_{a\in S_\tau}\sigma_a^2\Bigr)^{1/2},}
\]
and the claim follows from \codexedit{$\sigma_a\le\bar\sigma$} and $\codexedit{|S_\tau|=m_{\mathrm{kept}}}$.
The tolerance condition is essential: thresholding alone cannot bound arbitrary optimization error; the guarantee is for the SNR-thresholded, noise-tolerant calibration rule.
Moments excluded by $S_\tau$ are not counted in $\mathrm{CalGap}(S_\tau,\sigma)$ because the solver no longer matches them; their effect moves into $\mathrm{Approx}(\phi;S_\tau)$, which is the bias--variance tradeoff measured in the workload-dimension ablation.
\end{proof}
\end{credblock}

\section{Bias-Aware NA+MI Interval and the SNR Deployment Diagnostic}\label{app:bias-aware}
\begin{credblock}

{\color{green}This section gives the coverage-corrected interval motivated in} Appendix~\ref{app:discussion}\cred{; it turns the coverage--privacy paradox into a measurable diagnostic and a conservative correction}.
Throughout, $\widehat\tau_{\mathrm{NA+MI}}$ is the NA+MI point estimate, $\sigma_{\mathrm{NA+MI}}^2$ its Rubin-rules variance (which accounts for DP moment noise), and $\widehat{\mathrm{Approx}}(\phi)$ an estimate of the workload approximation bias defined below.

\begin{proposition}[Bias-aware NA+MI interval]
\label{prop:bias_aware}
Suppose the estimator admits the expansion
$\widehat\tau_{\mathrm{NA+MI}}-\tau = b_\phi+\sigma_{\mathrm{NA+MI}}Z+o_p(\sigma_{\mathrm{NA+MI}}+|b_\phi|)$ with $Z\Rightarrow\mathcal N(0,1)$, where $b_\phi$ is the leading workload approximation bias, and that the diagnostic is conservative: $\Pr\{|b_\phi|\le\widehat{\mathrm{Approx}}(\phi)+o_p(\sigma_{\mathrm{NA+MI}})\}\to 1$.
Define $\widetilde\sigma^2=\sigma_{\mathrm{NA+MI}}^2+\widehat{\mathrm{Approx}}(\phi)^2$ and $\mathrm{CI}_{\mathrm{corr}}(\alpha)=\widehat\tau_{\mathrm{NA+MI}}\pm z_{1-\alpha/2}\,\widetilde\sigma$.
Then for confidence levels with $z_{1-\alpha/2}\ge\sqrt{3}$ (including standard 95\% intervals), $\mathrm{CI}_{\mathrm{corr}}(\alpha)$ has asymptotic coverage at least $1-\alpha$.
When $b_\phi=0$ it reduces to the usual NA+MI interval; when $|b_\phi|\gg\sigma_{\mathrm{NA+MI}}$ it becomes conservative rather than collapsing around a biased center.
\end{proposition}

\begin{proof}
Condition on the event $|b_\phi|\le A+o_p(\sigma_{\mathrm{NA+MI}})$ with $A=\widehat{\mathrm{Approx}}(\phi)$.
Ignoring the lower-order term, coverage is bounded below by
$\Pr\{|\sigma_{\mathrm{NA+MI}}Z+b_\phi|\le z_{1-\alpha/2}\sqrt{\sigma_{\mathrm{NA+MI}}^2+A^2}\}$, which for fixed $A$ is minimized over $|b_\phi|\le A$ at $|b_\phi|=A$, where it equals
\[
f(a)=\Phi\bigl(zs(a)-a\bigr)+\Phi\bigl(zs(a)+a\bigr)-1,
\]
with $z=z_{1-\alpha/2}$, $a=A/\sigma_{\mathrm{NA+MI}}$, and $s(a)=\sqrt{1+a^2}$.
Differentiating,
\[
f'(a)=\phi\bigl(zs(a)-a\bigr)\Bigl(\tfrac{za}{s(a)}-1\Bigr)+\phi\bigl(zs(a)+a\bigr)\Bigl(\tfrac{za}{s(a)}+1\Bigr),
\]
where $\phi$ denotes the standard normal density.
If $za/s(a)\ge 1$ then $f'(a)\ge 0$ directly.
Otherwise $f'(a)\ge 0$ is equivalent to $\exp\{-2zas(a)\}\ge\frac{1-za/s(a)}{1+za/s(a)}$, i.e., writing $y=za/s(a)$, to $\operatorname{arctanh}(y)\ge y/(1-y^2/z^2)$.
For $z^2\ge 3$ this holds termwise: $\operatorname{arctanh}(y)=\sum_{k\ge0}y^{2k+1}/(2k+1)$, $y/(1-y^2/z^2)=\sum_{k\ge0}y^{2k+1}/z^{2k}$, and $z^{2k}\ge 3^k\ge 2k+1$ for $k\ge 1$.
Hence $f$ is nondecreasing on $a\ge0$ whenever $z\ge\sqrt{3}$, so $f(a)\ge f(0)=2\Phi(z)-1=1-\alpha$.
\end{proof}

\paragraph{Estimating $\widehat{\mathrm{Approx}}(\phi)$.}
Let $\widehat\psi_\phi(W)$ be the estimated orthogonal ATE score with nuisances fitted in the $\phi$ basis, centered at $\widehat\tau_{\mathrm{NA+MI}}$, and let $b(X)\in\R^r$ be a fixed diagnostic dictionary richer than $\phi$, scaled so that $\E[b(X)b(X)^\top]\approx I$ under the synthetic distribution.
With the residual score moment
$\widehat R_\phi=\bigl\lVert n_{\mathrm{syn}}^{-1}\sum_i\widehat\psi_\phi(W_i^{\mathrm{syn}})\,b(X_i^{\mathrm{syn}})\bigr\rVert_2$,
a practical conservative choice is $\widehat{\mathrm{Approx}}(\phi)=\widehat L_{\mathrm{diag}}\widehat R_\phi$, where $\widehat L_{\mathrm{diag}}$ is the same empirical Lipschitz/overlap constant used to score candidate features in \textsc{Causal-AIM}.
Large residual moments indicate that the current workload leaves systematic orthogonal-score structure unexplained{\color{green} in the regime at high }$\eps${\color{green} where} DP noise shrinks but approximation bias remains.
Both quantities are computable from the synthetic release and a public pilot or held-out covariate sample, so the ratio $\widehat{\mathrm{Approx}}(\phi)/\sigma_{\mathrm{NA+MI}}$ can be tracked at deployment.

\paragraph{Practical rule.}
Keep moment $j$ iff $\mathrm{SNR}_j:=|\widetilde q_j|/\sigma_j\ge\tau_{\mathrm{SNR}}$ (default $\tau_{\mathrm{SNR}}=3$) as the calibration filter, and report the corrected interval $\widehat\tau_{\mathrm{NA+MI}}\pm z_{1-\alpha/2}(\sigma_{\mathrm{NA+MI}}^2+\widehat{\mathrm{Approx}}(\phi)^2)^{1/2}$.
If $\widehat{\mathrm{Approx}}(\phi)\ll\sigma_{\mathrm{NA+MI}}$, the correction is negligible and DP noise dominates; if $\widehat{\mathrm{Approx}}(\phi)\gtrsim\sigma_{\mathrm{NA+MI}}$, the analysis is approximation-dominated{\color{green}, which is the diagnostic signature} of the coverage--privacy paradox at high $\eps$.

\paragraph{Connection to robust Bayes.}
The correction treats workload approximation as local model misspecification: NA+MI accounts for posterior uncertainty induced by DP noise, while the $\widehat{\mathrm{Approx}}(\phi)^2$ term enlarges the interval to cover discrepancy between the working workload model and the data-generating law.
\cred{This correction is the additive, workload-aware analogue of robust-Bayes and misspecification-aware posterior calibration \citep{bissiri2016genbayes,watson2016approxmodels,grunwald2017safebayes,miller2018coarsened}, which inflate posterior spread to remain valid under model misspecification.}
\end{credblock}

\begin{credblock}
\section{NA+MI: Detailed Pseudocode}\label{app:na-mi}

Algorithm~\ref{alg:nami} {\color{green}gives an implementation-level version of} the noise-aware multiple-imputation procedure of Section~\ref{sec:uq}.

\begin{algorithm}[h]
\caption{Noise-Aware Multiple Imputation (NA+MI)}
\label{alg:nami}
\begin{algorithmic}[1]
\Require \codexedit{Measured DP moments $\widetilde q_S$ on coordinates $S$}, mechanism noise covariance $\Sigma_S$, number of draws $M$, synthetic size $n_{\mathrm{syn}}$, level $\alpha$, optional SNR threshold $\tau_{\mathrm{SNR}}$
\State \codexedit{Compute the retained set $S_\tau=\{j\in S:|\widetilde q_j|/\sigma_j\ge\tau_{\mathrm{SNR}}\}$ once from $\widetilde q_S$ (take $S_\tau=S$ if thresholding is off)}
\For{$\codexedit{\ell}=1,\dots,M$}
    \State Draw $\codexedit{q_{S}^{(\ell)}\sim\N(\widetilde q_S,\Sigma_S)}$ \Comment{flat-prior Gaussian posterior on measured coordinates}
    \State $\codexedit{P^{\mathrm{syn},(\ell)}}\leftarrow$ max-entropy calibration of \eqref{eq:iprojection} on $\Pi_{S_\tau}\codexedit{q_{S}^{(\ell)}}$ (Algorithm~\ref{alg:synthesis})
    \State $\codexedit{D^{\mathrm{syn},(\ell)}}\leftarrow n_{\mathrm{syn}}$ ancestral-sampling draws from $\codexedit{P^{\mathrm{syn},(\ell)}}$
    \State $(\codexedit{\widehat\tau^{(\ell)},\widehat v^{(\ell)}})\leftarrow$ doubly robust ATE estimate and its variance estimate on $\codexedit{D^{\mathrm{syn},(\ell)}}$
\EndFor
\State $\bar\tau\leftarrow \codexedit{M^{-1}\sum_\ell\widehat\tau^{(\ell)}}$;\quad $W_M\leftarrow \codexedit{M^{-1}\sum_\ell\widehat v^{(\ell)}}$;\quad $B_M\leftarrow \codexedit{(M-1)^{-1}\sum_\ell(\widehat\tau^{(\ell)}-\bar\tau)^2}$
\State $T_M\leftarrow W_M+(1+1/M)B_M$;\quad $\nu_M\leftarrow (M-1)\bigl(1+\tfrac{W_M}{(1+1/M)B_M}\bigr)^2$ \Comment{Rubin's rules \citep{rubin1987mi}}
\Ensure Point estimate $\bar\tau$ and interval $\bar\tau\pm t_{\nu_M,1-\alpha/2}\sqrt{T_M}$
\end{algorithmic}
\end{algorithm}

The between-imputation component $B_M$ is what widens the interval as $\eps$ decreases: smaller $\eps$ means larger DP noise, more dispersed posterior draws $\codexedit{q_S^{(\ell)}}$, and hence more variable $\codexedit{\widehat\tau^{(\ell)}}$.
The bias-aware variant of Appendix~\ref{app:bias-aware} replaces $T_M$ by $T_M+\widehat{\mathrm{Approx}}(\phi)^2$.
\end{credblock}

\begin{credblock}
\section{Experimental Protocol and Supporting Results}\label{app:exp-details}

\begin{codexblock}

\subsection{Feature Construction and Defaults}

For each dataset, continuous covariates are discretized into $L=5$ quantile bins and the resulting indicator variables define $\phi$.
When quantile edges coincide, bins collapse, so the effective dimension is data-dependent; IHDP yields $p=48$ at $L=5$.
The small-$L$ default is consistent with practice in DP quantile estimation, where a handful of privately estimated quantiles per variable is reliable at moderate budgets \citep{gillenwater2021dpquantiles,kaplan2022dpapproxquantiles,lei2011dpmest}.

Unless stated otherwise, experiments use $n_{\mathrm{syn}}=10n$ synthetic records per release, $M=20$ MI draws, no max-entropy calibration ridge ($\alpha_{\mathrm{cal}}=0$), and no SNR thresholding in the main pipeline.
Outcomes are standardized by fixed public constants and clipped at $B=5$; constants and clip fractions are reported in Appendix~\ref{app:acs-dgp}.
The non-private oracle uses unstandardized, unclipped outcomes.
\cred{Estimated propensities in the DR analysis are clipped to $[0.02,0.98]$ (a conservative implementation bound, distinct from the positivity constant $\eta$ of Assumption~\ref{assump:causal} and from the DGP overlap parameter swept in the overlap ablation);} continuous-outcome DR nuisance fits use a fixed ridge linear model, and binary outcomes use logistic regression.
Main benchmark experiments use $n_{\mathrm{rep}}=500$ replications, adaptive and ACS studies use $100$, and ablations use $200$.
\dslp[red]{Every run logs its full configuration, code version, and data source.}

\subsection{Metrics}

All metrics are computed over $n_{\mathrm{rep}}$ independent replications.
Let $\widehat\tau^{(r)}$ be the ATE estimate, $\mathrm{CI}^{(r)}$ its 95\% confidence interval, and $\tau$ the true ATE.

\begin{itemize}
  \item \textbf{ATE bias}: $\mathrm{Bias}=\bigl|n_{\mathrm{rep}}^{-1}\sum_r\widehat\tau^{(r)}-\tau\bigr|$.
  \item \textbf{ATE RMSE}: $\mathrm{RMSE}=\{n_{\mathrm{rep}}^{-1}\sum_r(\widehat\tau^{(r)}-\tau)^2\}^{1/2}$.
  \item \textbf{CI coverage}: $\mathrm{Cov}=n_{\mathrm{rep}}^{-1}\sum_r\ind\{\tau\in\mathrm{CI}^{(r)}\}$, targeting the nominal 95\% level.
  \item \textbf{CI length}: $\mathrm{Len}=n_{\mathrm{rep}}^{-1}\sum_r|\mathrm{CI}^{(r)}|$.
  \item \textbf{Marginal fidelity}: $\mathrm{TVD}=d^{-1}\sum_{j=1}^d\mathrm{TV}(\hat P_j,\hat P_j^{\mathrm{syn}})$, the average total-variation distance across one-dimensional marginals.
\end{itemize}

\subsection{Experiment Grid and Ablations}

\begin{itemize}
  \item \textbf{Causal versus generic workloads}: compare ATE RMSE and bias for causal workload synthesis, MST, and AIM over $\eps\in\{0.5,1,2,5\}$ with $\delta=1/n^2$.
  \item \textbf{Coverage calibration}: compare naive intervals on generic and causal synthetic data against NA+MI intervals under the same privacy budgets.
  \item \textbf{Adaptive selection}: compare \textsc{Causal-AIM} against the fixed causal workload at the same total privacy budget, sweeping $K\in\{1,3,5,10\}$ adaptive rounds.
  \item \textbf{Workload dimension}: vary quantile-bin granularity $L\in\{3,5,10,20\}$, with and without SNR thresholding at $\tau_{\mathrm{SNR}}=3$.
  \item \textbf{MI draws}: vary $M\in\{5,10,20,50,100\}$.
  \item \textbf{Synthetic sample size}: vary $n_{\mathrm{syn}}/n\in\{1,2,5,10,50\}$.
  \item \textbf{Overlap}: vary the positivity constant $\eta$ in the ACIC DGP and measure the effect on RMSE and coverage.
\end{itemize}

\end{codexblock}

\end{credblock}

\section{ACS Semi-Synthetic DGP Details}\label{app:acs-dgp}

\begin{credblock}
\paragraph{Outcome standardization constants.}
For DP measurement, each dataset's outcome is standardized as $(Y-m_{\mathrm{pub}})/s_{\mathrm{pub}}$ using the fixed public constants below (treated as public benchmark metadata, consistent with the public clip bound $B$), then clipped at $B=5$; all reported estimates and intervals are converted back to original units.
The non-private oracle uses unstandardized, unclipped outcomes.
\begin{center}
\footnotesize
\begin{tabular}{lrrr}
\toprule
Dataset & $m_{\mathrm{pub}}$ & $s_{\mathrm{pub}}$ & clip fraction \\
\midrule
IHDP & 12.12 & 15.43 & \cred{$0.6\%$} \\
Twins & 0 & 1 & $0\%$ \\
ACIC & $-0.317$ & 5.706 & $0.1\%$ \\
LaLonde/NSW & 5{,}301 & 6{,}624 & $0.4\%$ \\
ACS & 0 & 1 & $1.4\%$ \\
\bottomrule
\end{tabular}
\end{center}
\end{credblock}

\cred{We use 20 demographic covariates from the 2018 California ACS via \texttt{folktables} \citep{ding2021retiring}: age, education, marital status, relationship, disability, employment status, citizenship, migration, Hispanic origin, ancestry, nativity, deafness, blindness, cognitive difficulty, sex, race, PUMA, state, class of worker, and place of birth; structurally missing categorical codes are encoded as their own category.}

\emph{Treatment assignment.}
Let $X_n$ denote column-standardized covariates.
We draw $\beta\sim\N(0,I_d)$, normalize to unit $\ell_2$ norm, and set the propensity score
$e(x) = \operatorname{expit}\!\bigl(\beta^\top x + 0.8\sin(x_1) - 0.4\,x_2 x_3\bigr)$,
\cred{rescaled to $[0.1, 0.9]$ to enforce overlap.}
Treatment is drawn as $T_i\sim\mathrm{Bernoulli}(e(X_i))$.

\emph{Outcome model.}
We draw $g\sim\N(0,I_d)$ (normalized) and define the baseline outcome
$\mu_0(x)=g^\top x + 0.6\cos(x_3) + 0.3(x_1^2 - x_2)$.
Individual treatment effects are
$\tau(x)=0.8 + 0.5\tanh(x_1) + 0.25\sin(x_4 - x_5)$.
Observed outcomes are $Y_i = \mu_0(X_i) + \tau(X_i)\,T_i + \varepsilon_i$ with $\varepsilon_i\sim\N(0,1.1^2)$.
\cred{The ground-truth ATE is the mean of the sampled potential-outcome differences, $\tau=n^{-1}\sum_i\{Y_i(1)-Y_i(0)\}$.}

\begin{credblock}
\section{Supplementary Experimental Figures}\label{app:figures}

All figures below supplement the main-text results in Section~\ref{sec:experiments}; they are ordered as benchmark diagnostics, ACS validation, design ablations, and the fidelity-versus-causal-utility diagnostic.
\cred{Bias patterns for the main benchmark comparison appear in Figure~\ref{fig:exp1_bias}.}

\begin{figure}[ht]
\centering
\includegraphics[width=\textwidth]{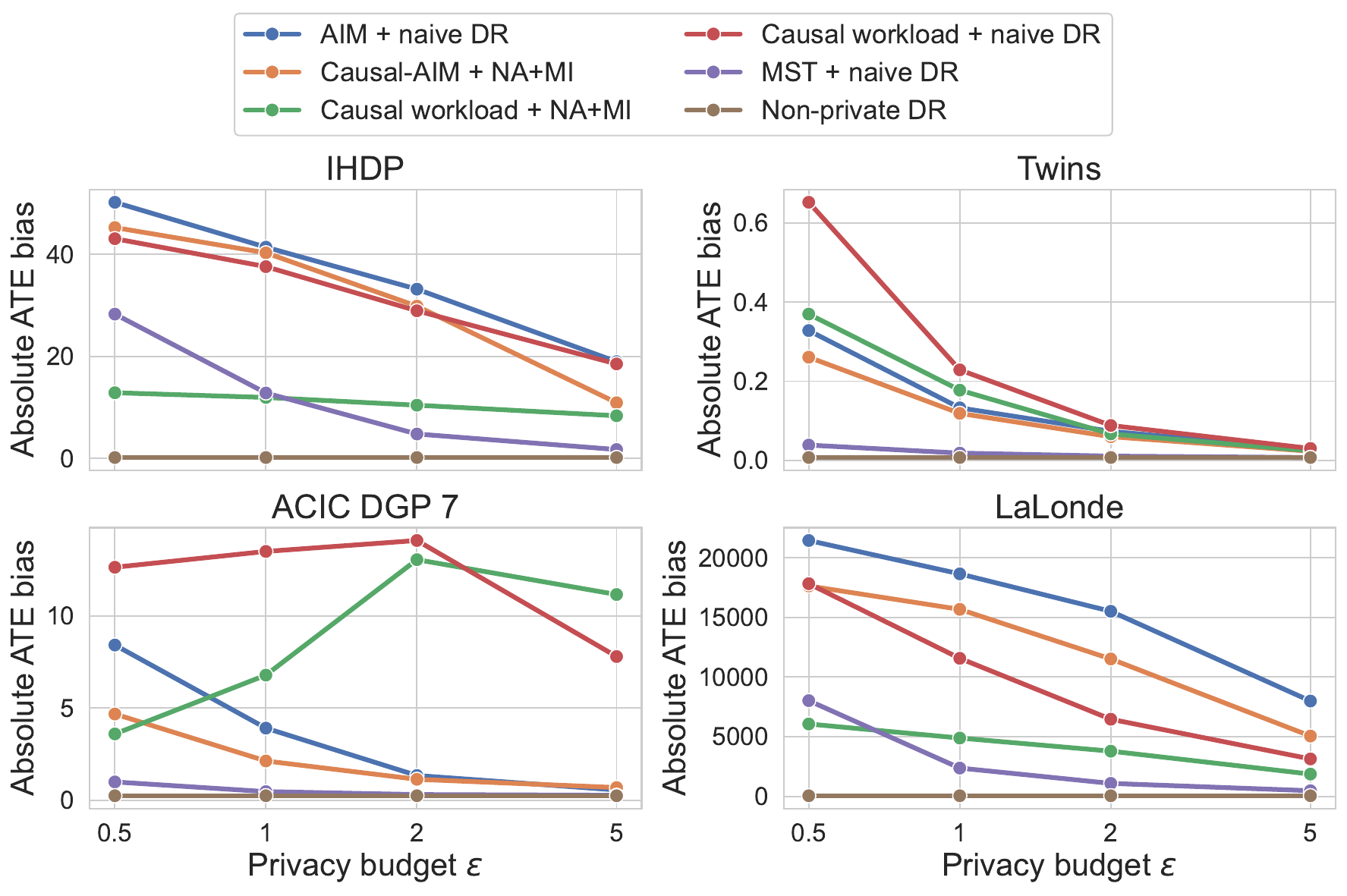}
\caption{\cred{Benchmark bias comparison:} \cred{Mean absolute ATE error ($n_{\mathrm{rep}}^{-1}\sum_r|\hat\tau^{(r)}-\tau|$)} across privacy budgets and four benchmark datasets ($n_{\mathrm{rep}}=500$).
Bias patterns by dataset and privacy level mirror the RMSE rankings in Figure~\ref{fig:exp1_rmse}.
\cred{At low $\eps$ ($\le 1$), all methods show substantial bias relative to the non-private oracle; MST + naive DR has the lowest bias on most datasets, \cred{with Causal + NA+MI lower on IHDP (and on LaLonde at $\eps=0.5$)}.}
As $\eps$ grows, all methods converge toward the oracle.}
\label{fig:exp1_bias}
\end{figure}

\begin{figure}[ht]
\centering
\includegraphics[width=\textwidth]{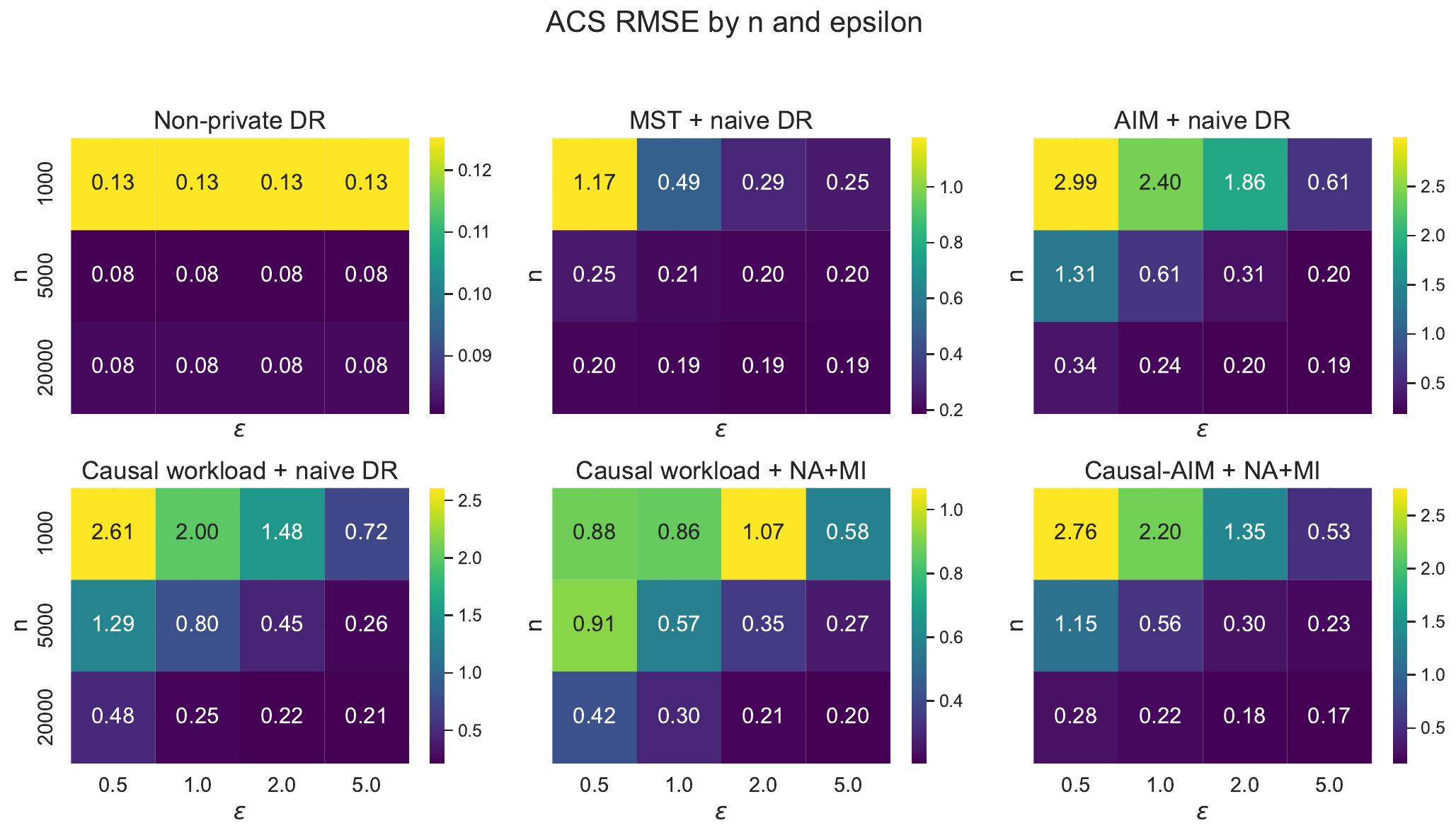}
\caption{\cred{ACS study:} ATE RMSE heatmaps on the ACS semi-synthetic study ($n_{\mathrm{rep}}=100$).
Rows index sample size $n\in\{1000,5000,20000\}$ and columns index $\eps\in\{0.5,1,2,5\}$.
Each panel shows one method.
At small $n$ and low $\eps$, all private methods have elevated RMSE; at large $n$ and $\eps$, all converge \cred{toward} the non-private oracle level.
\cred{NA+MI is competitive on RMSE at strict budgets (winning at $n{=}1000$, $\eps{=}0.5$) and is the only method with valid coverage (Figure~\ref{fig:exp4_acs_coverage}).}}
\label{fig:exp4_acs_rmse}
\end{figure}

\begin{figure}[ht]
\centering
\includegraphics[width=\textwidth]{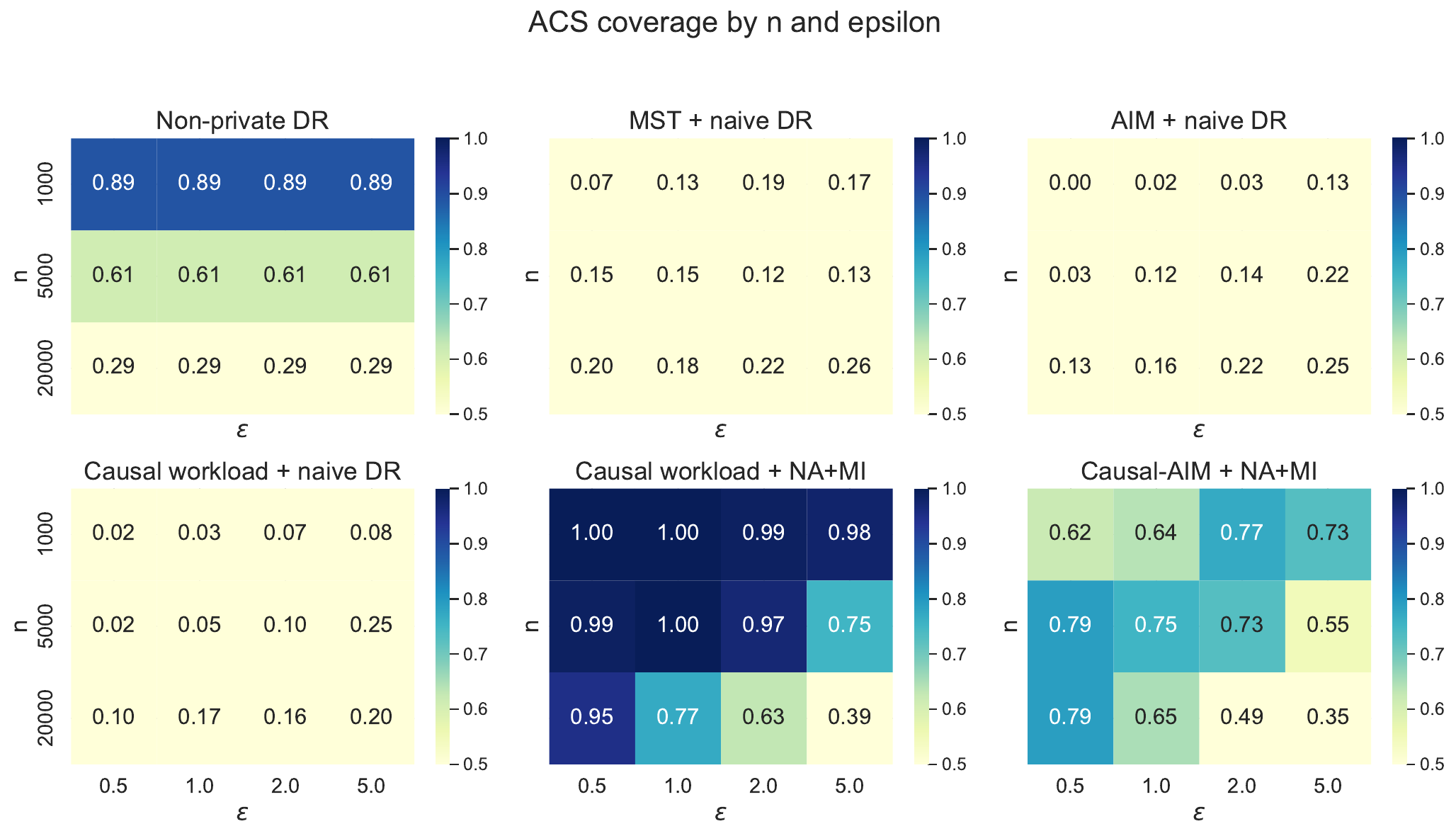}
\caption{\cred{ACS study:} Empirical 95\% CI coverage on the same ACS grid as Figure~\ref{fig:exp4_acs_rmse} ($n_{\mathrm{rep}}=100$).
Causal workload + NA+MI achieves coverage $1.00$ at $(n{=}1000,\,\eps{=}0.5)$, while MST + naive DR achieves only \cred{$0.07$}.
The coverage advantage of NA+MI is most pronounced at low $\eps$ and moderate $n$, \dslp[red]{where DP noise dominates the error}.}
\label{fig:exp4_acs_coverage}
\end{figure}

\begin{figure}[ht]
\centering
\includegraphics[width=0.78\textwidth]{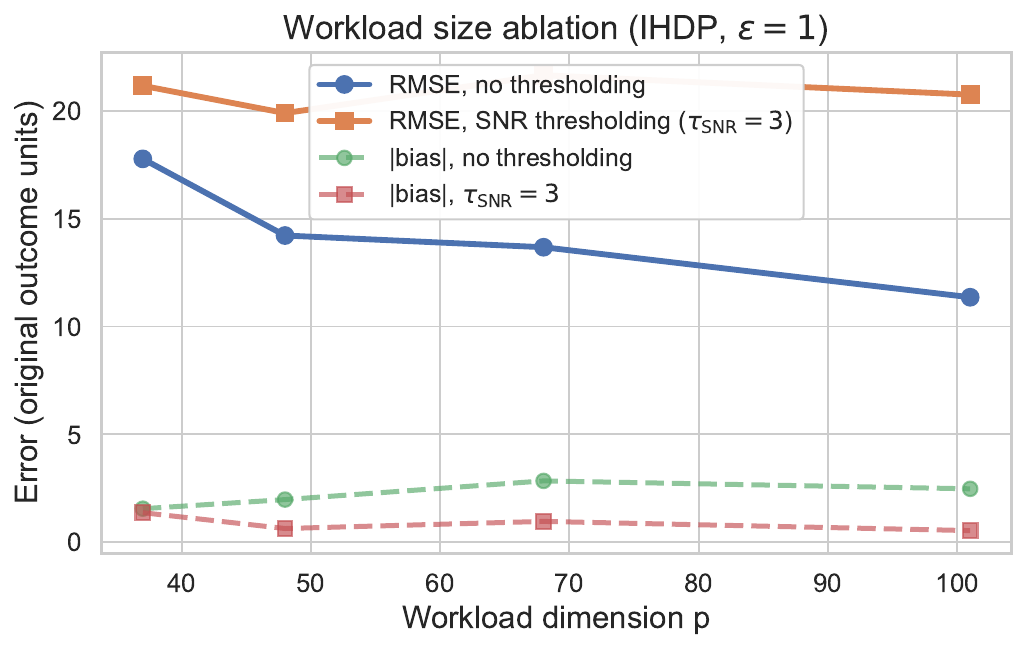}
\caption{\cred{Workload dimension on IHDP at $\eps=1$ ($n_{\mathrm{rep}}=200$), with and without SNR thresholding ($\tau_{\mathrm{SNR}}=3$). \cred{Without thresholding, RMSE decreases with $p$ in the tested range ($p\le 101$); with $\tau_{\mathrm{SNR}}=3$, RMSE is roughly flat, with lower bias and more variance}; \cred{the bias-dominated regime of Remark~\ref{rem:calibration_gap} is not reached on this dataset}.}}
\label{fig:ablation_dim}
\end{figure}

\begin{figure}[ht]
\centering
\includegraphics[width=\textwidth]{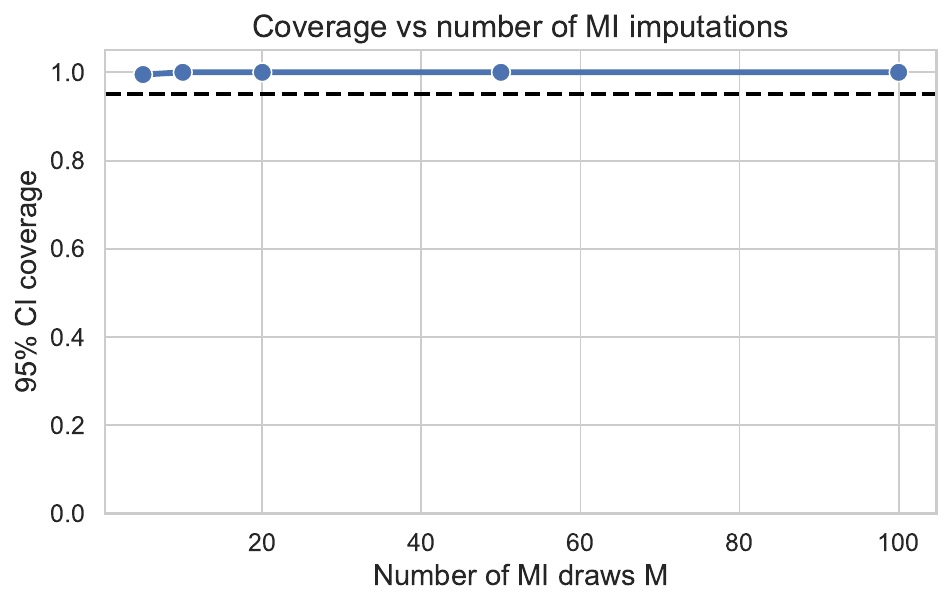}
\caption{MI draws ablation: empirical 95\% CI coverage versus number of imputation draws $M$ on IHDP at $\eps=1$ ($n_{\mathrm{rep}}=200$).
\cred{Coverage is at or above $0.995$ for every $M$ tested ($0.995$ at $M{=}5$, $1.000$ for $M\ge 10$): the interval is near-nominal already at small $M$.
We use $M{=}20$ as a conservative default for noise-aware MI.}}
\label{fig:ablation_mi}
\end{figure}

\begin{figure}[ht]
\centering
\includegraphics[width=\textwidth]{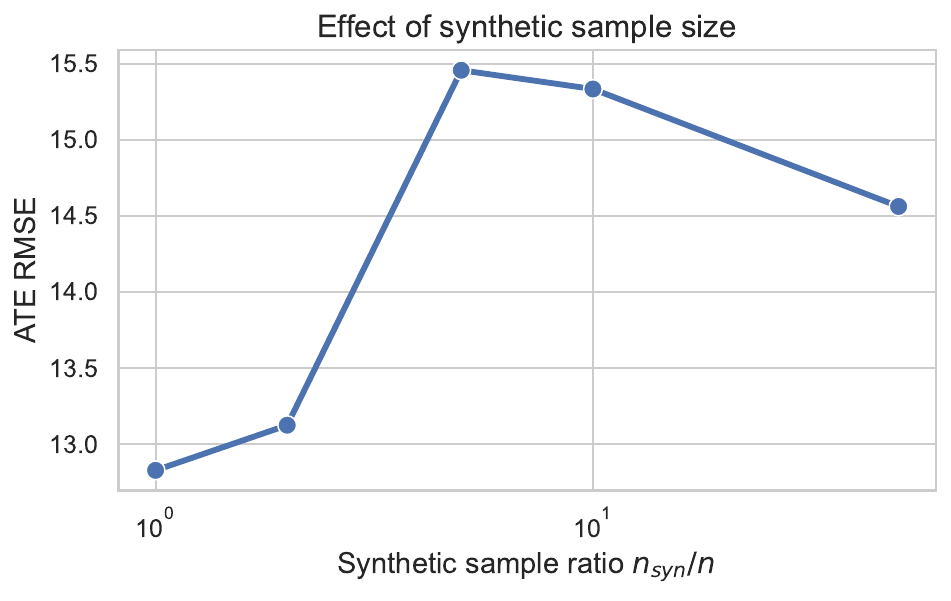}
\caption{Synthetic sample size ablation: ATE RMSE versus the ratio $n_{\mathrm{syn}}/n$ on IHDP at $\eps=1$ (\cred{$n_{\mathrm{rep}}=200$}).
\cred{RMSE varies modestly (range $12.8$--$15.5$) across ratios from $1$ to $50$, with no benefit from oversampling: the Monte Carlo term $O_p(n_{\mathrm{syn}}^{-1/2})$ in Theorem~\ref{thm:ate_bound} is dominated by DP noise already at $n_{\mathrm{syn}}=n$.
In practice, $n_{\mathrm{syn}}= n$ suffices.}}
\label{fig:ablation_nsyn}
\end{figure}

\begin{figure}[ht]
\centering
\includegraphics[width=\textwidth]{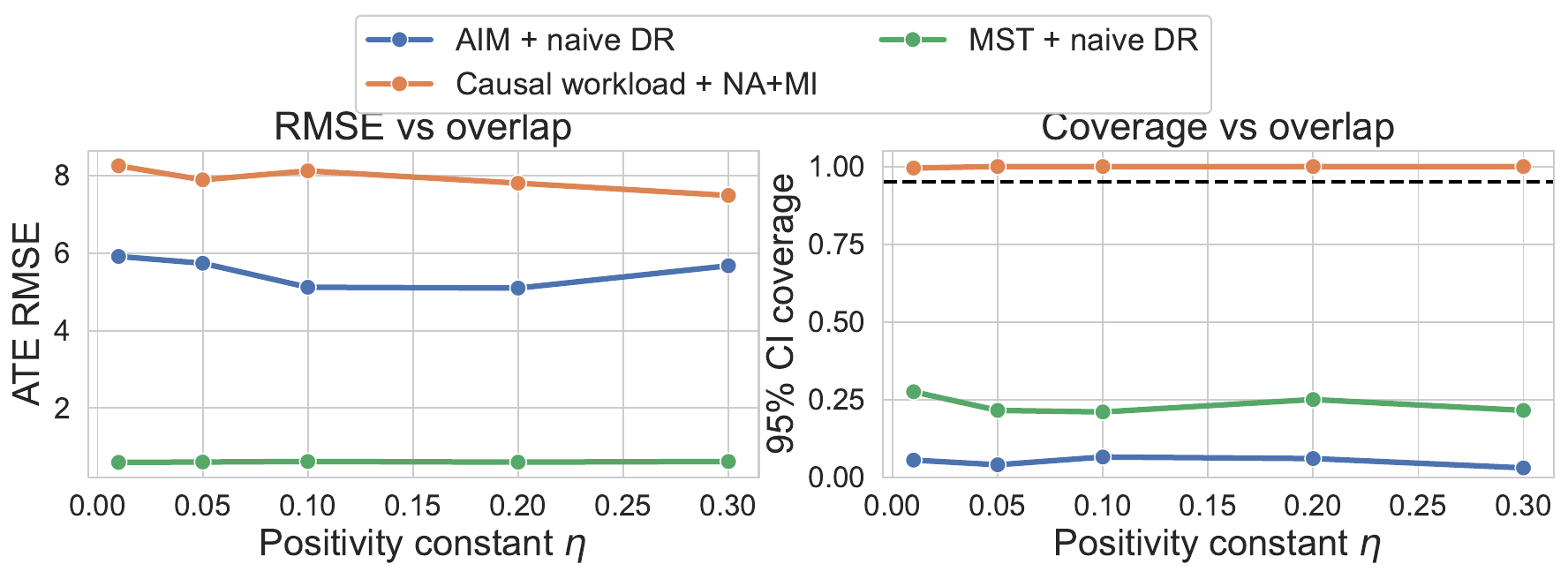}
\caption{Overlap ablation on ACIC DGP~7 at $\eps=1$ (\cred{$n_{\mathrm{rep}}=200$}).
The positivity constant $\eta$ varies from $0.3$ (strong overlap) to $0.01$ (near violation).
\cred{As overlap weakens, MST + naive DR maintains low RMSE ($0.60$--$0.62$) but poor coverage ($0.21$--$0.28$), while Causal workload + NA+MI has higher RMSE ($7.5$--$8.3$) with near-nominal coverage ($0.995$--$1.00$) at every $\eta$.}
This confirms that the coverage--RMSE tradeoff persists across overlap regimes and that NA+MI remains the only method with usable confidence intervals.}
\label{fig:ablation_overlap}
\end{figure}

\begin{figure}[ht]
\centering
\includegraphics[width=\textwidth]{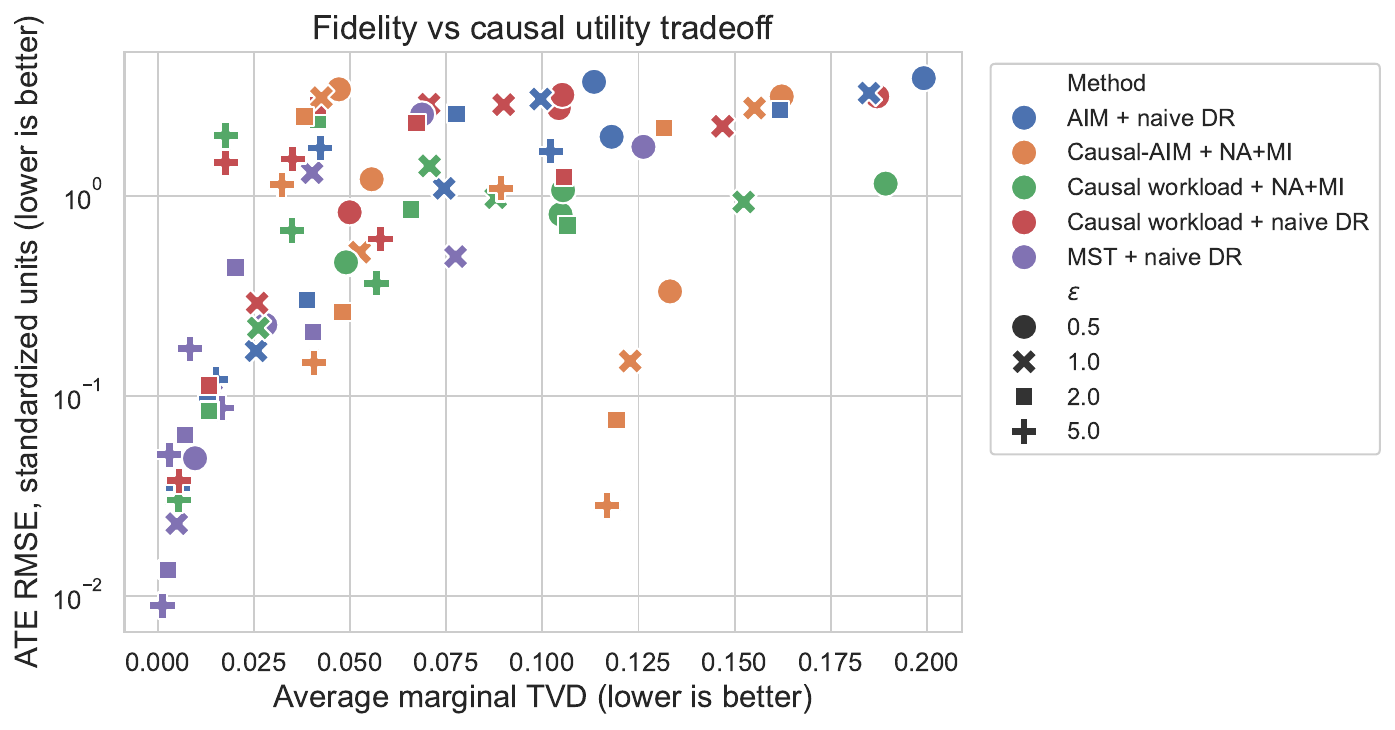}
\caption{Marginal fidelity versus causal utility across all method--dataset--$\eps$ configurations ($n_{\mathrm{rep}}=500$).
Each point represents one experimental cell; the $x$-axis is average per-marginal TVD and the $y$-axis is \cred{ATE RMSE in standardized outcome units (log scale), so datasets with different outcome scales are comparable}.
\cred{Generic workloads (MST) dominate on marginal fidelity (TVD) in nearly every configuration and on RMSE at $\eps\ge 2$; at $\eps=0.5$, causal workloads with NA+MI match or beat MST on RMSE on half the benchmarks (and on IHDP at $\eps=1$) despite worse TVD.
Distributional fidelity therefore does not order methods by causal utility; the case for causal workloads rests on noise-aware uncertainty quantification (Figure~\ref{fig:exp2_coverage}).}}
\label{fig:fidelity_tradeoff}
\end{figure}

\section{Supplementary Empirical Analyses}\label{app:promoted}

This appendix reports supplementary experiments; all use the defaults of Section~\ref{sec:experiments} at $n_{\mathrm{rep}}=200$ (scalability: $n_{\mathrm{rep}}=100$) unless noted.

\begin{codexblock}
\subsection{Multi-Estimand Reuse from One Synthetic Release}
From a \emph{single} DP synthetic release, we estimate the ATE, ATT, and a subgroup effect with NA+MI intervals and no additional privacy spending.
All eighteen dataset--budget--estimand cells attain nominal coverage (Table~\ref{tab:multi-estimand}).
\dslp[red]{Direct DP ATE estimators cannot offer this reuse: each answered query spends additional privacy budget.}

\begin{table}[!htbp]
\centering
\caption{Multi-estimand reuse from a single DP release: 95\% CI coverage (RMSE, original outcome units) of NA+MI intervals for ATE, ATT, and a subgroup effect, all computed from one synthetic dataset per replication ($n_{\mathrm{rep}}=200$).}
\label{tab:multi-estimand}
\begin{tabular}{llccc}
\toprule
Dataset & $\eps$ & ATE & ATT & Subgroup \\
\midrule
IHDP & 0.5 & 1.00 (15.98) & 1.00 (17.48) & 1.00 (15.44) \\
IHDP & 1 & 1.00 (15.63) & 1.00 (18.31) & 1.00 (14.96) \\
ACIC & 0.5 & 1.00 (4.36) & 1.00 (6.79) & 1.00 (4.45) \\
ACIC & 1 & 1.00 (7.57) & 1.00 (9.42) & 1.00 (7.59) \\
LaLonde/NSW & 0.5 & 1.00 (7{,}613) & 1.00 (8{,}042) & 1.00 (7{,}840) \\
LaLonde/NSW & 1 & 1.00 (6{,}688) & 1.00 (6{,}945) & 1.00 (6{,}849) \\
\bottomrule
\end{tabular}
\end{table}
\end{codexblock}

\subsection{Direct DP ATE Proxy Comparison}
\begin{table}[!htbp]
\centering
\caption{RMSE and coverage against output-perturbation \emph{proxies} of direct DP ATE estimators. The proxies apply DP only at the output rather than to the propensity fit, a weaker DP regime than the published methods \citep{schroder2025private,ohnishi2024covbal}; these results should therefore be read as a favorable contextual approximation of the direct-estimator family. Cells are RMSE / coverage, $n_{\mathrm{rep}}=200$.}
\label{tab:direct-dp}
\begin{tabular}{llccc}
\toprule
Dataset & $\eps$ & PrivATE-style & OA-style & Causal+NA+MI \\
\midrule
IHDP & 0.5 & 41.8 / 0.97 & 39.8 / 0.96 & 16.4 / 1.00 \\
IHDP & 1 & 18.6 / 0.98 & 22.0 / 0.96 & 14.8 / 1.00 \\
ACIC & 0.5 & 2.96 / 0.94 & 2.74 / 0.97 & 4.27 / 1.00 \\
ACIC & 1 & 1.47 / 0.93 & 1.34 / 0.96 & 8.30 / 1.00 \\
\bottomrule
\end{tabular}
\end{table}
The comparison (Table~\ref{tab:direct-dp}) is {\color{green}mixed}: the proxies win on single-estimand RMSE on ACIC, while Causal + NA+MI wins on IHDP and attains full coverage everywhere; and the synthetic-data release supports multiple estimands at shared privacy cost (Table~\ref{tab:multi-estimand}).

\subsection{Hybrid Causal and Generic Workloads}
A 50/50 split of the privacy budget between causal moments and generic one-way marginals ($n_{\mathrm{rep}}=200$; RMSE / coverage):
\begin{center}
\footnotesize
\begin{tabular}{llccc}
\toprule
Dataset & $\eps$ & Causal+NA+MI & Hybrid+NA+MI & MST + naive DR \\
\midrule
IHDP & 0.5 & 17.6 / 1.00 & 19.2 / 1.00 & 36.2 / 0.03 \\
IHDP & 1 & 14.8 / 0.99 & 18.7 / 1.00 & 18.5 / 0.01 \\
ACIC & 0.5 & 4.48 / 1.00 & 5.79 / 1.00 & 1.20 / 0.14 \\
ACIC & 1 & 8.13 / 1.00 & \textbf{5.17 / 1.00} & 0.64 / 0.22 \\
\bottomrule
\end{tabular}
\end{center}
The hybrid Pareto-improves on the fixed causal workload on ACIC at $\eps=1$ (lower RMSE at equal coverage) but trails on IHDP; it is an alternative when both point accuracy and coverage matter, not a new default.

\subsection{Adaptive Selection Operating Point}
On IHDP (\cred{$K\in\{1,2,3,5,10\}$,} $n_{\mathrm{rep}}=200$; RMSE / coverage), no $K$ recovers the fixed workload's calibration:
\begin{center}
\footnotesize
\begin{tabular}{lccccc}
\toprule
$\eps$ & $K{=}1$ & $K{=}2$ & $K{=}3$ & $K{=}5$ & $K{=}10$ \\
\midrule
0.5 & 42.2 / 0.02 & 50.7 / 0.11 & 50.8 / 0.20 & 50.3 / 0.31 & 49.1 / 0.60 \\
1 & 29.1 / 0.04 & 37.2 / 0.11 & 42.1 / 0.16 & 43.1 / 0.38 & 42.5 / 0.61 \\
2 & 14.6 / 0.17 & 22.7 / 0.14 & 31.5 / 0.11 & 34.4 / 0.44 & 39.8 / 0.64 \\
\bottomrule
\end{tabular}
\end{center}
Contrast with ACIC (Figure~\ref{fig:exp3_adaptive}): whether \textsc{Causal-AIM} helps is dataset-dependent\dslp[red]{, which is the operating-point message of Section~\ref{sec:experiments}}.

\subsection{Scalability in Covariate Dimension}
On ACS at $n=5000$, $\eps=1$, and $n_{\mathrm{rep}}=100$, increasing the number of covariates raises both RMSE and synthesis time, with mirror descent on the $\phi$-basis as the dominant cost.
\cred{The $d=40$ and $d=60$ configurations extend the 20 base ACS covariates with deterministic derived features (squares, pairwise interactions, and threshold indicators).}
\begin{center}
\footnotesize
\begin{tabular}{lcc}
\toprule
Covariates $d$ & RMSE & Median synthesis time \\
\midrule
20 & 0.56 & 2.9s \\
40 & 1.07 & 5.1s \\
60 & 1.38 & 8.5s \\
\bottomrule
\end{tabular}
\end{center}

\subsection{Calibration-Ridge Sensitivity}
On IHDP at $\eps=1$ ($n_{\mathrm{rep}}=200$), RMSE / coverage across \codexedit{max-entropy calibration ridge values $\alpha_{\mathrm{cal}}\in\{0.001,0.01,0.1,1.0\}$}: $14.9/1.00$, $14.2/1.00$, $15.3/1.00$, $18.6/1.00$.
\codexedit{Performance is insensitive to this solver regularization} over three orders of magnitude (mild degradation only at $\codexedit{\alpha_{\mathrm{cal}}=1}$); \codexedit{this ablation is distinct from the theoretical ridge parameter $\lambda$ in Corollary~\ref{cor:noise_ridge}, and SNR thresholding remains the main calibration filter (Remark~\ref{rem:calibration_gap}).}
\end{credblock}

\section{Additional Discussion, Limitations, and Future Work}\label{app:discussion}

\begin{codexblock}
\paragraph{Why generic fidelity is insufficient.}
Causal inference requires preserving conditional relationships such as $Y\mid T,X$, not only low-dimensional distributional marginals.
A synthetic dataset that matches $P(X,T)$ and $P(X,Y)$ but distorts $P(Y\mid T,X)$ can have good marginal fidelity and poor ATE accuracy.
The causal workload makes the target-specific requirement explicit by measuring the moments that identify the orthogonal score in the chosen feature basis.

\paragraph{Practical guidance.}
A practitioner should specify the target estimand, derive the corresponding orthogonal score and moment requirements, choose a feature map $\phi$ that approximates those moments, generate synthetic data with a fixed causal workload or \textsc{Causal-AIM}, and report NA+MI uncertainty.
On our benchmarks, richer feature maps helped throughout the tested range ($L$ up to 20 bins on IHDP), so the practical default is the richest $\phi$ the budget supports, using SNR thresholding and the measured $\mathrm{CalGap}$ as safeguards when moments approach the noise floor.

\paragraph{Coverage--RMSE tradeoff.}
The experiments reveal a tradeoff: MST often leads on point RMSE at moderate-to-loose budgets, but its naive confidence intervals are far below nominal coverage at strict budgets.
Causal workload + NA+MI is the conservative recommendation when calibrated uncertainty quantification is essential.
\textsc{Causal-AIM} + NA+MI is useful when lower RMSE is prioritized and the operating point is favorable, while hybrid causal + generic workloads can help when both point accuracy and coverage matter (Appendix~\ref{app:promoted}).

\paragraph{Coverage--privacy paradox.}
NA+MI coverage can decrease as $\eps$ increases because DP noise shrinks while workload approximation bias remains.
In ACIC, Causal + NA+MI coverage drops from $1.00$ at $\eps\le1$ to $0.90$ at $\eps=2$ and $0.17$ at $\eps=5$, while the other three main benchmarks remain at $0.99{-}1.00$.
The diagnostic ratio $\widehat{\mathrm{Approx}}(\phi)/\sigma_{\mathrm{NA+MI}}$ identifies this approximation-dominated regime, and Appendix~\ref{app:bias-aware} proves a conservative bias-aware interval that inflates variance by $\widehat{\mathrm{Approx}}(\phi)^2$.

\paragraph{Limitations.}
The implementation instantiates the framework on the AIM/Private-PGM family, with MST as the generic comparator; the workload-design principle should transfer to other select--measure--reconstruct generators, but we have not evaluated modern generator families beyond this class.
The feature map $\phi$ and discretization of continuous covariates introduce approximation error, especially when the true outcome surface is highly nonlinear.
Theorem~\ref{thm:ate_bound} gives a useful moment-level decomposition, but constants can be loose and the theorem does not fully analyze every downstream fitted-nuisance routine run on sampled synthetic rows.
Finally, valid NA+MI intervals can be wide at strict privacy budgets, especially on small benchmarks such as LaLonde/NSW.

\paragraph{Future work.}
Natural extensions include structured graphical-model solvers and convex relaxations for better calibration; continuous covariates through private kernel mean embeddings or random Fourier features; subgroup-specific and CATE workloads; automated DP-aware selection of $\phi$; and information-theoretic lower bounds linking overlap, workload dimension, and unavoidable privacy distortion.
\end{codexblock}

\end{document}